\documentclass[11pt, a4paper, logo, copyright]{googledeepmind}

\pdftrailerid{redacted}

\usepackage{natbib}
\setcitestyle{numbers,open={[},close={]}}
\usepackage[utf8]{inputenc} % allow utf-8 input
\usepackage[T1]{fontenc}    % use 8-bit T1 fonts
\usepackage{hyperref}       % hyperlinks
\usepackage{url}            % simple URL typesetting
\usepackage{booktabs}       % professional-quality tables
\usepackage{kantlipsum, lipsum}
\usepackage{xcolor}
\definecolor{darkblue}{rgb}{0, 0, 0.5}
\hypersetup{colorlinks=true, citecolor=darkblue, linkcolor=darkblue, urlcolor=darkblue}
\usepackage{etoc}
\usepackage{dsfont}
\usepackage{gdm-colors}

\usepackage{amssymb,amsmath,amsthm,bbm}
\usepackage{verbatim,float,url,dsfont}
\usepackage{graphicx,subfigure,psfrag}
\usepackage{algorithm, algpseudocode}
\usepackage{algorithmicx}
\usepackage{bibentry}
\algtext*{EndWhile}% Remove "end while" text
\algtext*{EndIf}% Remove "end if" text
\algtext*{EndFor}% Remove "end if" text
\algtext*{EndProcedure}% Remove "end if" text
\usepackage{mathtools,enumitem}
\usepackage{multirow}
\usepackage{ragged2e}
\usepackage{xr-hyper}
\usepackage{array}
\usepackage{stmaryrd,scalerel}
\usepackage[dvipsnames]{xcolor}
\usepackage{tikz}
\usepackage[utf8]{inputenc} % allow utf-8 input

\usepackage{booktabs}       % professional-quality tables
\usepackage{nicefrac}         % compact symbols for 1/2, etc.
\usepackage{microtype}      % microtypography
\usepackage{fancyhdr}

\newtheoremstyle{mytheoremstyle}
{.5em}  % Space above theorem (e.g., 1em, 12pt, 0.5\baselineskip)
{.5em}  % Space below theorem
{\itshape} % Body font (e.g., \itshape for italics, \normalfont for roman)
{}     % Indent amount (empty = no indent, \parindent = paragraph indent)
{\bfseries} % Theorem head font (e.g., \bfseries for bold)
{.}    % Punctuation after theorem head
{.5em} % Space after theorem head (e.g., .5em or " ")
{}     % Theorem head spec (can be left empty, meaning 'normal')
\theoremstyle{mytheoremstyle}
\newtheorem{theorem}{Theorem}

\newtheorem{corollary}[theorem]{Corollary}
\newtheorem{proposition}[theorem]{Proposition}
\newtheorem{remark}[theorem]{Remark}

\newcommand{\argmin}{\mathop{\mathrm{argmin}}}

\newcommand{\minimize}{\mathop{\mathrm{minimize}}}
\newcommand{\maximize}{\mathop{\mathrm{maximize}}}
\newcommand{\st}{\mathop{\mathrm{subject\,\,to}}}
\newcommand{\Bern}{\mathop{\mathrm{Bern}}}
\DeclareMathOperator{\Tr}{Tr}

\def\R{\mathbb{R}}

\def\N{\mathbb{N}}
\def\E{\mathbb{E}}
\def\P{\mathbb{P}}
\def\T{\mathsf{T}}

\def\Var{\mathrm{Var}}

\def\cN{\mathcal{N}}

\def\cX{\mathcal{X}}

\def\ind#1{\mathbbm{1}\left\{#1\right\}}

\usepackage{mathtools}

\usepackage{mathtools}
\mathtoolsset{showonlyrefs}\usepackage{afterpage}

\newcommand{\Error}{\mathsf{Error}}
\newcommand{\ErrorRatio}{\mathsf{ErrorRatio}}
\newcommand{\Cost}{\mathsf{Cost}}
\newcommand{\MSE}{\mathrm{MSE}}
\def\costh{{c}_h}
\def\costg{{c}_g}

\newcommand{\eqd}{\stackrel{\textnormal{d}}{=}}

\newcommand{\aseq}{\stackrel{\textnormal{a.s.}}{=}}
\newcommand{\asgeq}{\stackrel{\textnormal{a.s.}}{\geq}}

\newcommand{\iidsim}{\stackrel{\textnormal{i.i.d.}}{\sim}}

\usepackage{epigraph} 

\usepackage[most,skins,theorems]{tcolorbox}

\tcbset{
  aibox/.style={
    width=\linewidth,
    top=6pt,
    bottom=0pt,
    colback=blue!6!white,
    colframe=black,
    colbacktitle=black,
    enhanced,
    center,
    attach boxed title to top left={yshift=-0.1in,xshift=0.15in},
    boxed title style={boxrule=0pt,colframe=white,},
  }
}
\newtcolorbox{AIbox}[2][]{aibox,title=#2,#1}

\tcbuselibrary{listings, breakable}
\tcbset{
    graybox/.style={
        colback=gray!20,  % Background color (light gray)
        colframe=black,   % Border color (black)
        boxrule=1pt,      % Border thickness
        arc=4mm,          % Box corner rounding
        width=\textwidth, % Box width to match the text width
        before=\vspace{10pt}, % Space before the box
        after=\vspace{10pt},  % Space after the box
        center             % Center the box
    }
}
\title{Cost-Optimal Active AI Model Evaluation}
\author[1$\dagger$*]{Anastasios N. Angelopoulos}
\author[2]{Jacob Eisenstein}
\author[3]{Jonathan Berant}
\author[3]{Alekh Agarwal}
\author[2*]{Adam Fisch}

\affil[1]{University of California, Berkeley}
\affil[2]{Google DeepMind}
\affil[3]{Google Research}
\affil[$\dagger$]{Work done as an intern at GDM}
\affil[*]{Equal contribution.}

\begin{abstract}
The development lifecycle of generative AI systems requires continual evaluation, data acquisition, and annotation, which is costly in both resources and time. In practice, rapid iteration often makes it necessary to rely on synthetic annotation data because of the low cost, despite the potential for substantial bias.
In this paper, we develop novel, cost-aware methods for actively balancing the use of a cheap, but often inaccurate, {weak rater}---such as a model-based autorater that is designed to automatically assess the quality of generated content---with a more expensive, but also more accurate, {strong rater} alternative such as a human.
More specifically, the goal of our approach is to produce a low variance, unbiased estimate of the mean of the target "strong" rating, subject to some total annotation budget.
Building on recent work in active and prediction-powered statistical inference, we derive a family of cost-optimal policies for allocating a given annotation budget between weak and strong raters so as to maximize statistical efficiency. 
Using synthetic and real-world data, we empirically characterize the conditions under which these policies yield improvements over prior methods. We find that, especially in tasks where there is high variability in the difficulty of examples, our policies can achieve the same estimation precision at a far lower total annotation budget than standard evaluation methods.\looseness=-1

\end{abstract}

\usepackage[parfill]{parskip}
\begin{document}
\maketitle

\section{Introduction}
\label{sec:intro}

Accurately and efficiently evaluating generative AI  (GenAI) systems is a core technical challenge,  both for model development and for reliable model deployment. 
In this paper, we introduce new statistical tools for active, cost-sensitive model evaluation. 
Specifically, we develop evaluation pipelines that dynamically collect and annotate data using a mix of weak and strong annotation options in a way that is aware of their relative costs and strengths. 
The core idea is to strategically balance inexpensive but potentially inaccurate annotations from a \textit{weak rater} against  more accurate, but also more costly, annotations from a more sophisticated \textit{strong rater} alternative. 
The exact composition of the weak and strong raters is flexible; for example, the weak rater might be a small AI model or rule-based heuristic, while the strong rater might be a larger AI model, an AI model equipped with tools or larger inference-time reasoning capabilities, a human, or even an expert human. 
The cost of the evaluation might then be  measured in compute, latency, or dollars.
Active evaluation aims to minimize cost by selectively obtaining expensive annotations only when they are informative, relying on the cheaper option otherwise. All of the annotations are then combined using statistically principled, unbiased methods to yield  reliable, yet more cost-effective, performance metrics.\looseness=-1

Combining different data sources to improve evaluation quality is not new: in particular, the use of cheap but biased metrics as control variates to improve statistical efficiency in model evaluation has been explored before from various perspectives~\cite{angelopoulos2023prediction, angelopoulos2023ppi++, boyeau2024autoeval, chaganty-etal-2018-price, chatzi2024ppirank,  fisch2024stratified, jung2025trust, saad-falcon-etal-2024-ares, zrnic2024active}.  
Here, our main technical contribution is a theoretical framework for cost-optimal active evaluations---evaluation algorithms that strategically choose when to deploy the strong rater as opposed to the weak rater in order to minimize the final total cost of evaluation while remaining unbiased.
Informally, these  policies solve the following constrained optimization problem:\looseness=-1
\begin{align}
        \maximize \quad & \text{Accuracy of the evaluation,} \\
        \textrm{subject to} \quad & \text{Cost of the evaluation remaining below a budget }B.
    \label{eq:informal-optimizesc-cost}
\end{align}
The policies that solve this optimization problem, as we will see, depend both on (i) the relative costs of the raters, and (ii) the error of the weak rater with respect to the strong one. 

We derive these optimal policies via new technical extensions and combinations of modern techniques in statistics, namely, active statistical inference~\cite{zrnic2024active} and  prediction-powered inference~\cite{angelopoulos2023prediction, angelopoulos2023ppi++, zrnic2025cross}.
Building on this  foundation, we test and analyze empirical approximations to these optimal policies.
Our experiments demonstrate promising results for implementing  cost-optimal active evals in practice, although we also highlight important practical challenges that  naturally arise due to "cold-start" issues, as well as  imperfections of existing autorater models and their uncertainty estimates.\looseness=-1

\paragraph{Related work.} Prediction-powered inference (PPI)~\cite{angelopoulos2023prediction, angelopoulos2023ppi++, zrnic2025cross} is the technique of combining a small number of trusted observations  with predictions from a machine learning system for the purpose of statistical estimation.
Its core statistical principles are closely related to control variate estimators~\cite{ chaganty-etal-2018-price, Ripley87} as well as semi-parametric inference with missing data~\cite{chernozhukov2018double, aipw_robins, tsiatis2006missing}. 
Recently, a body of work has explored applying PPI to the evaluation of GenAI systems, where human annotations are combined with "autorater" outputs~\cite{boyeau2024autoeval,  chatzi2024ppirank, egami2023design, fisch2024stratified, saad-falcon-etal-2024-ares};  though it has also been noted that the sample efficiency gained is limited when the autorater is not sufficiently accurate~\cite{dorner2025limits, thakur2025judgingjudgesevaluatingalignment}.
A natural extension of PPI is  to actively select a fixed number of examples on which to obtain trusted observations, while deferring the remaining examples to the autorater~\cite{gligoric2024can, zrnic2024active}. 
Roughly speaking, these approaches sample human annotations with probability proportional to the uncertainty of the autorater. 
However, they work only in a restricted setting in which the ratio of expensive to cheap ratings, $n/N$, is fixed in advance, and then pick the optimal policy subject to that constraint.
No guidance is given as to what this ratio should be based on the relative costs of the ratings, or even what the total number of examples $N$ should be.\looseness=-1 

\paragraph{Contributions.}
Our work extends this literature both theoretically and empirically. 
Our core theoretical contribution is the derivation of error-minimizing sampling rules under cost constraints.
That is, previous methods have a fixed ratio $n/N$ and a policy that maximizes accuracy under that fixed ratio, while our policy maximizes accuracy subject to a cost constraint by optimizing everything including the ratio $n/N$.
We derive two forms of optimal policies: (i) the best fixed sampling rate (Proposition~\ref{prop:optimal-random}), and (ii) the best active sampling rule that depends on covariates (Proposition~\ref{prop:optimal-budget}). 
One additional novelty of our work is that it improves upon the policy proposed by ~\citet{zrnic2024active} by accounting for the constraint that the policy must lie in $[0,1]$ for all values of $x$.
Finally, Appendix~\ref{app:additional-theory} includes further theoretical innovations, such as an extension to convex M-estimators and an optimal method for selecting the covariate $x$ (as opposed to the label, as considered in the prior work).\looseness=-1

On the empirical end, we extend the scope of the standard PPI framework to heterogeneous model evaluation settings involving two distinct rating sources, each with a different cost-performance profile. 
This goes beyond the typical "human-vs-LLM" scenario described above, and encompasses any situation where less expensive, less accurate ratings are combined with more expensive, more accurate ones, even if both sources are automated (e.g., smaller vs. larger models or more vs. less inference-time reasoning). In Sections~\ref{sec:synthetic_data} and \ref{sec:estimating-policies}, we present an extensive empirical investigation into the conditions under which these new sampling rules prove beneficial over classical estimation. Specifically, we identify that the success of our framework is determined by: (a) the overall error of the weak rater, (b) the overall variance of the target strong rater, and (c) the heteroskedasticity of the weak rater’s errors.\looseness=-1

\section{Cost-optimal annotation policies}
\label{sec:optimal-policies}

We now describe our methods for constructing active, cost-optimal evals.
The methods rely on one critical ingredient: an \emph{annotation policy} $\pi$. 
The job of the annotation policy is to look at the input  and decide whether it should be labeled by the expensive rater.
The theory in this section derives the optimal policies under different restrictions on the policy space.
These policies are \emph{oracle} policies---they depend on properties of the data distribution, some of which are impossible to know in advance.
The point of this section is to tell us what policies we should be targeting, not how to find them; later, we will  explore how to estimate them in practice.\looseness=-1

\subsection{Basic notation}
We observe inputs $X \sim P_X$ from some space $\mathcal{X}$ and distribution $P_X$: in the setting of LLMs, we think of the input $X$ as containing the prompt as well as the response from one or multiple LLMs.
Our goal is to approximate an expensive rating signal $h(X) \in \R$, such as a preference label from a human or a large AI model, with a cheap automated evaluator $g(X) \in \R$;
for notational convenience we introduce $H \triangleq h(X)$  and $G \triangleq g(X)$.
In our setup, querying $h$ costs $\costh$, and querying $g$ costs $\costg$, where $\costg < \costh$.
We seek to query $h$ only when it is ``worth the cost''.

We  consider a sequential setting: for every $t \in \N$, we observe $X_t$ drawn i.i.d. from the distribution $P_X$ and $G_t \sim P_{G \mid X}$. Upon observing $X_t$, we then have the option of also querying $H_t \sim P_{H \mid X}$.
Our objective is to estimate $\theta^* = \E[H]$, the mean target rating.
To this end, we develop estimators that efficiently sample only the data points for which $H_t$ is needed, and stop sampling after a certain budget is exhausted.
Define the binary random variable $\xi_t \sim \Bern(\pi_t(X_t))$, which is the indicator of whether we sampled $H_t$; it takes on the value $1$ with probability $\pi_t(X_t)$, and we have the freedom to choose the annotation policy $\pi_t$ based on the previous data we have seen so far.
We  estimate $\theta^*$ with the following unbiased estimator, defined for all $T \in \N$:
\begin{equation}
    \label{eq:base-active-estimator}
    \hat\theta_T = \frac{1}{T}\sum\limits_{t=1}^T \Delta_t \quad \text{ where } \Delta_t = G_t + \left(H_t - G_t\right) \frac{\xi_t}{\pi_t(X_t)}.
\end{equation}
Here $\pi_t \in \Pi$ for some policy class $\Pi$. 
%\alekh{This is not clear. How are $\pi$ and $\pi_t$ related? Is $\pi = (\pi_1,\ldots,\pi_T)$ a collection of maps, or is it an algorithm? If neither, then I don't follow how a static policy $\pi$ maps to the RHS.}
If $\Pi$ is left unspecified, it should be assumed that $\pi_t$ can be any function with range $(0,1]$. This is the sequential active inference estimator from~\citet{zrnic2024active}: the difference here will be in how we set $\pi_t$ to balance the different labeling costs.
In general, the annotation policy $\pi_t$ is allowed to change arbitrarily online as a function of past data, as is the predictor $g$. 
The estimator is written in a sequential form, wherein the data points arrive one-by-one and we have the option of querying $h$ per-data-point.
For simplicity, here we will focus on the setting where the parameters of $\pi$ and $g$ remain fixed throughout and are \emph{not} updated online, or as if we are updating in batches; however our results will also hold asymptotically when $\pi$ and $g$ are updated online and converge.
To this end, we use the notation $\hat\theta_T^\pi$ to denote the estimator in \eqref{eq:base-active-estimator} with a fixed policy $\pi$, i.e.,  where $\pi_t = \pi$, $\forall t \in T$.
%\alekh{I think our optimal benchmark is a fixed policy, but actual policies that we try are sequentially adapted? A bit like how we evaluate things in online learning.}
%
To calculate the cost and error of our estimator, we  additionally define:\looseness=-1
\begin{align}
    \label{eq:error_t}
    \Error_T(\pi) &\triangleq \E\left[\left(\hat\theta_T^\pi - \theta^*\right)^2\right] =\frac{1}{T}\left(\Var(H) - \E[(H-G)^2] + \E\left[(H - G)^2\frac{1}{\pi(X)}\right]\right)
    % \\
    %     \Cost_T(\pi) &\triangleq T(c_h\E\left[\pi(X)\right]+c_g).
\end{align}
and
\begin{equation}
\Cost_T(\pi) \triangleq T(c_h\E\left[\pi(X)\right]+c_g).
\end{equation}

These functions describe the mean squared error and expected cost of the estimator with annotation policy $\pi$ as a function of time, and our goal will be to minimize one subject to a constraint on the other.
When we refer to a budget on the cost, it will be denoted as $B$.
Furthermore, we note that, for convenience, the cost-optimized policies that we present in the remainder of this section will relax the constraint that the stopping time $T^{\rm stop}$ at which $\Cost_{T^{\rm stop}}(\pi)$ is just under budget must be an integer, though this does not have a significant effect on the optimization for large enough budgets $B$ where $T^{\rm stop} \gg 1$. Some additional treatment for this restriction is included in Appendix~\ref{app:additional-theory}.

\subsection{Optimal random annotation}
\label{sec:optimal-random}
The simplest annotation policy does not depend on $X$, and simply queries $H$ with some fixed probability, which we  denote as $\pi(x)=p$ for a sampling rate $p \in \R$ and all $x \in \cX$.
In other words, we let $\pi \in \Pi^{\rm random} = \{ x \mapsto p : p \in (0,1] \}$.
When $p$ is too large, the cost of the estimator is too high; when $p$ is too small, the error blows up.
Our job is to choose the optimal balance, and the next proposition shows it has a simple, explicit form that depends both on the cost ratio $\costg / \costh$ and how the mean squared error of $G$ relative to $H$, that is ${\MSE(H, G) = \E[(H - G)^2]}$, compares to the  variance of $H$.\looseness=-1

\begin{proposition}
    \label{prop:optimal-random}
    Let $(X_1, G_1, H_1), \ldots, (X_T, G_T, H_T)$, $T \in \N$, be an i.i.d. sequence of real-valued random variables with joint distribution $P$, and define $\Error$, $\Cost$, and $\Pi^{\rm random}$ as above.
    Assume that $\P(G_1 = H_1) < 1$ and that $c_h > c_g > 0$, and define the optimization problem
    \begin{equation}
        %\begin{aligned}
            \minimize_{\substack{\pi \in \Pi^{\rm random}, ~T^{\rm stop} \in \R_{>0}}} \quad \Error_{T^{\rm stop}}(\pi) \quad
            \st \quad  \Cost_{T^{\rm stop}}(\pi) \leq B.
        %\end{aligned}
        \label{problem:optimal-random}
    \end{equation}
    Then the solution to Problem~\eqref{problem:optimal-random} for all $x \in \cX$ is
    \begin{equation}
        \label{eq:policy-optimal-random}
        \pi_\mathrm{random}(x) = \begin{cases}
            \sqrt{\frac{c_g}{c_h}\frac{\E[(H-G)^2]}{{\Var(H)} - \E[(H-G)^2]}} & \text{ if } \E[(H-G)^2] < \frac{\costh}{\costh + \costg}\Var(H) \\
            1 & \text{ otherwise.}
        \end{cases}
    \end{equation}    
\end{proposition}
% \begin{remark}
% Note that in actuality, the stopping time $T^{\rm stop}$ must be an integer, though this does not have a significant effect on the optimization for large enough budgets $B$. See  Proposition~\ref{prop:optimal-random-integer} in Appendix~\ref{app:additional-theory} for the solution with this additional restriction.
% \end{remark} 

We can make a few observations about $\pi_\mathrm{random}$. First, if the mean squared error of the weak rater $G$ is greater than the variance of $H$ (or more precisely, more than a $\costh/(\costh + \costg)$ fraction of the variance of $H$), then it is not helpful---and we should simply choose to query $H$ all  the time. If $\MSE(H, G)$ is sufficiently low, however, then the rate at which we sample $H$ varies inversely with both the ratio of $\Var(H)$ to $\MSE(H, G)$ and  the ratio of the cost of $H$ to the cost of $G$. This makes intuitive sense: if the target label $H$ is high variance but our "weak" rater $G$ is in fact a  fairly "strong" rater (in that it produces similar ratings to those of $H$), then we should primarily exploit $G$'s low cost, high-quality predictions, while sampling $H$ at just a  low rate to correct for any minor bias that arises.\looseness=-1

\subsection{Optimal active annotation}
\label{sec:optimal-active}
Next, we study policies that are allowed to \emph{depend} on $X$; that is, they can elect to query $H$ with some probability that depends on $X$.
This strategy can greatly improve statistical power when the error distribution is heteroskedastic in $X$; for example, when some prompts are much harder than others.
In this setting, it makes sense for $\pi$ to depend on $X$, and to ask for advanced rating help more often when $G$ is likely to be wrong.
Towards that end, we define our annotation policy class to be $\pi \in \Pi = \{ x \mapsto f(x) : f(x) \in (0,1]; \forall x \in \cX \}$, which is the set of annotation policies placing a strictly positive amount of sampling mass on each query. 
As the next proposition shows, the optimal policy will depend on the uncertainty of the weak rater, $u(x) \triangleq \mathbb{E}[(H - G)^2 \mid X = x]$, expressed as the expected mean squared conditional error given $X = x$. For notational convenience, we also define the random variable $U \triangleq u(X)$.

\begin{proposition}
    \label{prop:optimal-budget}
    In the same setting as Proposition~\ref{prop:optimal-random}, define $\Pi$ as above, let $\cX$ be discrete, and additionally define the optimization problem
    \begin{align}
            \minimize_{\substack{\pi \in \Pi, ~T^{\rm stop} \in \R_{> 0}}} \quad  \Error_{T^{\rm stop}}(\pi) \quad
            \st \quad  \Cost_{T^{\rm stop}}(\pi) \leq B.
        \label{problem:optimal-budget}
    \end{align}
    Define the scaled and clipped policy, $\pi_\mathrm{clip}$, as:
    \begin{equation}
        \pi_\mathrm{clip}(x; \tau) = \min\left(\gamma^*(\tau) \sqrt{u(x)}, 1\right) = \begin{cases}
            \gamma^*(\tau) \sqrt{u(x)} &  \text{ if } \sqrt{u(x)} \leq \tau \\ 
            1 & \text{otherwise,}
        \end{cases}
    \end{equation}
    where $\costh > \costg > 0$ and $\gamma^*(\tau) \in \big(0, \frac{1}{\tau}\big]$ is defined as
    \begin{equation}
        \gamma^*(\tau) = \min\left(
        \sqrt{\frac{{c_g}/{c_h} + \mathbb{P}\left(U > \tau^2\right)}{\big(\mathrm{Var}(H) - \mathbb{E}[U \ind{U \leq \tau^2}]\big)_+}},
        \frac{1}{\tau} \right).
    \end{equation}
    Then the solution to Problem~\eqref{problem:optimal-budget} is $\pi_\mathrm{active}(x) = \pi_\mathrm{clip}(x;  \tau^*)$,
    where $\tau^* > 0$ is the  solution to
    \begin{equation}
    \label{eq:tau-optimization}
    \tau^* = \argmin_{\tau \in \mathbb{R}_{>0}}\left(\costh\mathbb{E}[\pi_\mathrm{clip}(x; \tau)] + \costg \right) \left(\Var(H) + \mathbb{E}\left[U \left(\pi_\mathrm{clip}(x; \tau)^{-1}  - 1\right)\right]\right).
    \end{equation}
\end{proposition}
\begin{remark} The final optimization problem presented for the clipping threshold $\tau^*$ is non-convex and has no analytical solution.
However, because it is a 1-dimensional optimization problem, we can coarsely discretize and optimize $\tau$ via simple grid search.\looseness=-1
\end{remark}

On a technical level, the solution in Proposition~\ref{prop:optimal-budget}  has a similar form to the active sequential estimator proposed in \citet{zrnic2024active}, but with an optimized proportionality constant, as well as additional clipping to rigorously account for the constraints on $\pi(x) \in (0, 1]$. The latter point is particularly important, as it is not accounted for in prior work. In contrast to the fixed, prespecified ratio prescribed by prior work, in Appendix~\ref{app:active-cases} we show how the \emph{cost-optimal} target ratio of expensive to cheap ratings can be as extreme as $0$ or $1$, depending on the cost ratio of $G$ to $H$.

While the form of $\pi_\mathrm{active}$ is more complex than that of $\pi_\mathrm{random}$, it still admits a fairly straightforward interpretation: for some confidence threshold $\tau^*$ below which the conditional mean squared error of $G$ over all confident data points with $\sqrt{u(x)} \leq \tau^*$  is sufficiently low, we sample proportional to $\sqrt{u(x)}$. On the remaining highly uncertain examples where $\sqrt{u(x)} > \tau^*$,  we  always use $H$, and ignore $G$. The exact threshold $\tau^*$ depends on the distributions of $H$ and $G$, and their cost-ratio.

 Proposition~\ref{prop:optimal-budget} is also a direct generalization of Proposition~\ref{prop:optimal-random}.
When $X$ is independent of $(H-G)^2$ so that $u(x) = \E[(H - G)^2]$ $\forall x \in \mathcal{X}$,  we see that $\pi_\mathrm{active}$ reduces to $\pi_\mathrm{random}$:  
\begin{equation}
    \underbrace{\gamma^*(\tau^*) \sqrt{u(x)}}_{\text{optimal active}} = \sqrt{\frac{\costg}{\costh}\frac{ \E[(H-G)^2 \mid X = x] }{\Var(H) - \E[(H-G)^2]}} =  \underbrace{\sqrt{\frac{c_g}{c_h} \frac{\E[(H-G)^2]}{\Var(H) - \E[(H-G)^2]}}}_{\text{optimal random}}.
\end{equation}
The intuitive conclusion is that active querying can help if the conditional squared error of $U$ has significant variance to it (i.e., there exist some regions of $\mathcal{X}$ where $G$ has a much higher level of agreement with $H$ than on other regions of $\mathcal{X}$, such as on easy vs. hard examples). This can be contrasted with the optimal random policy, $\pi_\mathrm{random}$, from \eqref{eq:policy-optimal-random}: there we sample at a fixed rate for each $X$, where that rate  depends only  on $G$'s \emph{average} error with respect to $H$ across all types of inputs.\looseness=-1

\begin{AIbox}{Takeaways: Cost-optimal annotation policies}

We characterize two types policies for sampling the  expensive rating $H$ given a  budget $B$: $\pi_\mathrm{random}$ chooses the optimal fixed  probability $p^* \in (0, 1]$, while $\pi_\mathrm{active}$ defines an optimal input conditional probability $\pi_\mathrm{active}(x) \in (0, 1]$. Both  navigate the following trade-off:  reducing $\mathbb{E}[\pi(X)]$ increases the total number of $X$ samples we can afford to rate at all, but not querying $H$ when $G$ is inaccurate increases variance. Finally, both policies converge to the baseline  estimator (i.e., $\pi_\mathrm{base}(x) = 1$) when the  error of $G$ is too high relative to the  variance of $H$.\looseness=-1
\end{AIbox}

\section{Comparing cost-optimal annotation policies in simulated settings}
\label{sec:synthetic_data}

The estimation error of the optimal policies presented in Section~\ref{sec:optimal-policies} depends on the distributions of the  expensive target label $H$, the cheap estimated label $G$, and the cost-ratio $\costg / \costh$ for querying $G$ versus $H$. To build a clearer  understanding of how these variables influence the performance of our proposed policies, we now conduct a series of carefully controlled experiments on simulated data. Note that since all of the key distributional quantities (i.e., $\Var(H)$, $\MSE(H, G)$, etc) are  known  in the synthetic settings we  consider in this section, we are also able to compute $\pi_\mathrm{active}$ and $\pi_\mathrm{random}$ exactly---as opposed to the more difficult  real-world data settings we will tackle in Section~\ref{sec:estimating-policies}. %This allows us to isolate these specific distributional aspects before examining the performance of these  policies on real-world data in Section~\ref{sec:estimating-policies}.

\subsection{Metrics}
%The core distributional properties we are interested in analyzing the effects of are  $\Var(H)$, $\MSE(H, G)$, and $\Var(U)$. 
To measure the relative performance of annotation policy $\pi_1$ vs $\pi_2$, we  compute the \emph{ratio} of their errors at $T^{\rm stop}_i$. %= \left\lfloor \frac{B}{\costh \E[\pi_i(X)] + \costg}\right\rfloor$. 
 Once again relaxing the restriction that $T^{\rm stop}_i \in \mathbb{N}$, %\jacob{I don't think this relaxation is actually mentioned in S2}
%or assuming that $B$ is sufficiently large such that $T^{\rm stop}_i =\left\lfloor \frac{B}{\costh \E[\pi_i(X)] + \costg}\right\rfloor \approx  \frac{B}{\costh \E[\pi_i(X)] + \costg}$, 
we  compute a budget-free approximation based on  the expression for $\Error_{T_i^\mathrm{stop}}(\pi)$, where $T_i^\mathrm{stop} = B / (\costh\E[\pi_i(X)] + \costg)$:
\begin{equation}
\begin{split}
    \ErrorRatio(\pi_1, \pi_2) \triangleq \frac
    {\big(\costh\E[\pi_1(X)] + \costg\big)
    \left(\Var(H) - \E[(H - G)^2] + \E\left[(H - G)^2 \frac{1}{\pi_1(X)}\right]\right)}
    {\big(\costh\E[\pi_2(X)] + \costg\big)
    \left(\Var(H) - \E[(H - G)^2] + \E\left[(H - G)^2 \frac{1}{\pi_2(X)}\right]\right)}.
\end{split}
\end{equation}
Note that while $\ErrorRatio(\pi_1, \pi_2)$ does not depend on the budget, it does implicitly depend on $P_X$ as well as $P_{H \mid X}$ and $P_{G \mid X}$.
We will focus on $\ErrorRatio(\pi_{\rm active}, \pi_{\rm base})$, the error ratio of the active estimator  to the baseline estimator which only uses  $H$, as well as $\ErrorRatio(\pi_{\rm active}, \pi_{\rm random})$, which compares the active estimator to the estimator that doe not depend on $X$. Note that for $\ErrorRatio(\cdot, \pi_\mathrm{base})$, we disregard  $\costg$ for $\pi_\mathrm{base}$, and replace the denominator with $\costh\Var(H)$.\looseness=-1

\subsection{Gaussian data}
\label{sec:gaussian_data}

We construct an experiment where we change $\Var(H)$, $\MSE(H, G)$, and $\Var(U)$ independently (recall that we introduce $U \triangleq u(X) = \E[(H - G)^2 \mid X]$ in  Section~\ref{sec:optimal-active}). First, we draw $H \sim \mathcal{N}(0, \nu)$ from a normal distribution with variance $\nu$. We then draw $U$ from a gamma distribution with mean $\mu$ and variance $\eta$ so that $\E[U] = \mathrm{MSE}(H, G) = \mu$, and $\Var(U) = \eta$. This can be satisfied by sampling $U \sim \mathrm{Gamma}\left(\mu^2/\eta, \eta/{\mu}\right)$. Finally, we set $G = H + \sqrt{U}$. 

\begin{figure}[!t]
    \centering
    \includegraphics[width=1\linewidth]{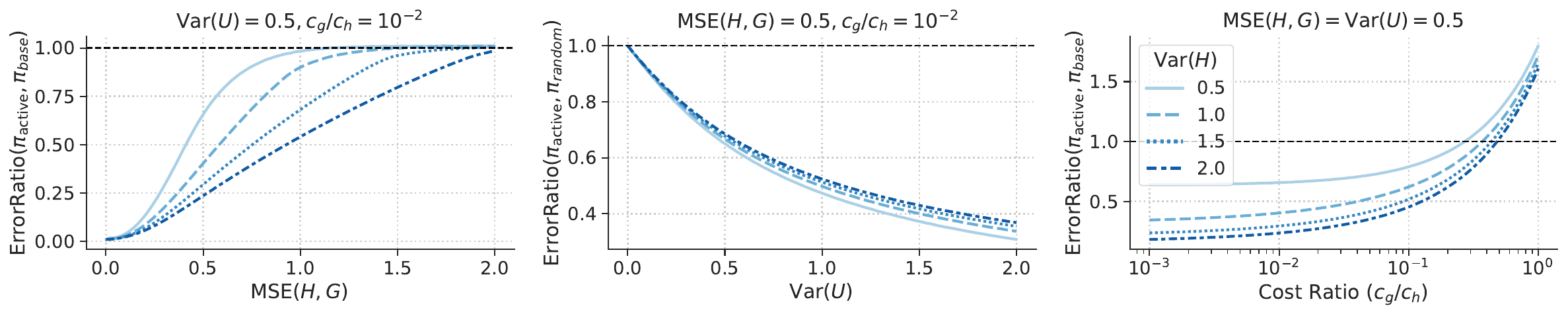}
    \includegraphics[width=1\linewidth]{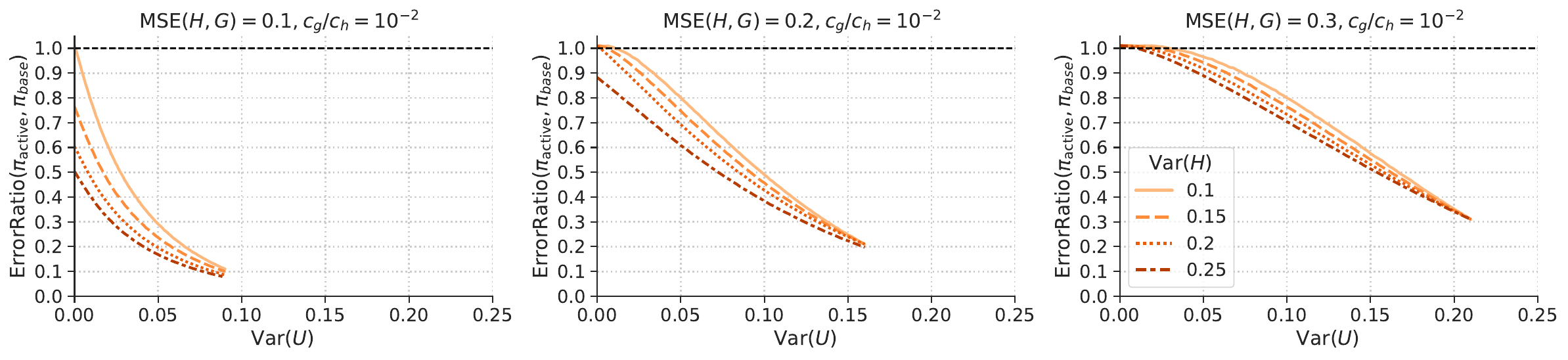}
    \vspace{-14pt}
    \caption{Top: Results on the Gaussian data (\S\ref{sec:gaussian_data}) while varying  $\MSE(H, G)$,  $\Var(U)$, and $\costg/\costh$. Bottom: Results on the bounded Bernoulli data (\S\ref{sec:bernoulli_data}) while varying  $\MSE(H, G)$  and $\Var(U)$. Each line plots a different value of $\Var(H)$, where we choose values  that are representative of low, medium, or high variance settings  compared to $\MSE(H, G)$. In the Bernoulli setting (bottom), $\MSE(H, G)$ is equivalent to the error rate of $G$, and $\Var(U)$ can be at most $\MSE(H, G)(1 - \MSE(H, G))$.}
    \label{fig:gaussian}
   \vspace{-10pt}
\end{figure}

Results are shown in the top row of Figure~\ref{fig:gaussian}. The left panel illustrates the error ratio of $\pi_{\mathrm{active}}$ to $\pi_{\mathrm{base}}$ as a function of $\MSE(H,G)$, while keeping $\Var(U)$ fixed at $0.5$. Each line plots results for a different level of $\Var(H)$. As expected, the error of $\pi_\mathrm{active}$  increases with the $\MSE(H,G)$, with the rate of increase significantly influenced by $\Var(H)$. When $\MSE(H, G)$ is large relative to $\Var(H)$, $\pi_\mathrm{active}$ provides no benefit over $\pi_{\mathrm{base}}$. The middle panel plots the error ratio of $\pi_\mathrm{active}$ to $\pi_\mathrm{random}$ while varying $\Var(U)$ for a fixed $\MSE(H, G)$. For small values of $\Var(U)$, the conditional error in $G$ is nearly the same everywhere, and  there is no benefit to using $\pi_\mathrm{active}$ over $\pi_\mathrm{random}$. Larger values of $\Var(U)$, however, lead to  a large performance advantage for $\pi_\mathrm{active}$. Finally, the right panel plots the error ratio of $\pi_\mathrm{active}$ to $\pi_\mathrm{base}$ while keeping $\MSE(H, G)$ and $\Var(U)$ fixed, but varying the cost-ratio $\costg / \costh$. As expected, $\pi_\mathrm{active}$ is most effective when $\costg \ll \costh$.

\subsection{Bernoulli data}
\label{sec:bernoulli_data}

While the Gaussian setting above is informative, in many typical situations $H$ is bounded, such as when $H$ is a binary, Bernoulli rating for win-rate or accuracy estimation. This creates a more difficult setting for $\pi_\mathrm{active}$, since both $\Var(H)$ and $\Var(U)$ are upper-bounded by $0.25$ for Bernoulli $H$. In fact, in binary settings, $\MSE(H, G)$ and $\Var(U)$ are in tension: the more accurate $G$ is, the lower the variance of its errors, and $\pi_\mathrm{active}$ will be limited in terms of any relative benefit it can provide over $\pi_\mathrm{random}$. The same is also true for when $G$ is uniformly \emph{inaccurate}. 
To better analyze this kind of setting, we construct a controlled dataset similar to the previous one, but for binary $H$. First, we draw $H$ from a Bernoulli distribution with variance $\nu$, which is satisfied by $H \sim \mathrm{Bern}(0.5 + \sqrt{0.25 - \nu})$. Next, we draw $U$ from a Beta distribution with mean $\mu$ and variance $\eta$, where $\eta \leq \mu(1 - \mu)$, which is satified by $U \sim \mathrm{Beta}(\kappa \mu, \kappa (1 - \mu))$ for $\kappa = \frac{\mu(1 - \mu)}{\eta} - 1$. Finally, we flip $H$ with probability $U$ to get the prediction $G$ (i.e., $G$ is also Bernoulli with $\MSE(H, G) = \mu$).

Results are shown in the bottom row of Figure~\ref{fig:gaussian} for $\pi_\mathrm{active}$ vs $\pi_\mathrm{base}$ (see Appendix~\ref{app:additional_results}  for $\pi_\mathrm{active}$ vs. $\pi_\mathrm{random}$). As in the Gaussian setting,  the error ratio of $\pi_\mathrm{active}$ to $\pi_\mathrm{base}$ 
 improves dramatically with larger $\Var(U)$. Note that the active and random estimator are the same when $\Var(U) = 0$ .  For larger $\MSE(H, G)$,  $\Var(U)$ must also be increasingly large for $\pi_\mathrm{active}$  to improve significantly over $\pi_\mathrm{base}$. Indeed, on the right-hand side of the bottom row of Figure~\ref{fig:gaussian} where $\MSE(H, G)>\Var(H)$, we can see that  $\pi_\mathrm{random}$ provides no benefits over  $\pi_\mathrm{base}$; 
 that is,  $\mathrm{ErrorRatio}(\pi_\mathrm{random}, \pi_\mathrm{base}) = 1$ when $\Var(U) = 0$, which corresponds to the fixed-rate sampling policy as noted earlier.
 When $\Var(U) \gg 0 $, however,  $\pi_\mathrm{active}$ can obtain substantially lower estimation error than  $\pi_\mathrm{base}$. Still, unlike the earlier Gaussian data, the best active  error ratio in this setting is bounded from below by $\MSE(H, G)$, and is achieved when $U$ has maximum variance (which is also bounded).\looseness=-1%\footnote{The best case is in the limit where $U$ is binary. Here, the active policy will choose to label all points where $U = 1$, and use $G$ where $U = 0$, resulting in a sample efficiency gain of $(\MSE(H, G) + \costg/\costh) / (1 + \costg / \costh) > \MSE(H, G)$.}

\begin{AIbox}{Takeaways: Performance characteristics of cost-optimal annotation policies}
%As shown in Section~\ref{sec:optimal-policies}, cost-optimal policies are distribution dependent. 
In general, the following properties hold for active annotation  versus standard annotation (similar findings for random): (i) as the error of $G$, $\MSE(H, G)$, increases, the  benefit \textbf{decreases}; (ii) as the variance of the conditional squared-error of $G$, $\Var(U)$, increases, the  benefit \textbf{increases}; and (iii) as the cost ratio, $\costg/\costh$, of  $G$ relative to $H$ increases, the  benefit \textbf{decreases}.\looseness=-1
%Active annotation is  also more effective when $(H - G)^2$ can be  large, versus when it is bounded to a small range (e.g., Bernoulli data). \jacob{Is this takeaway independent from the point about $\Var(U)$?}\looseness=-1  %When errors are bounded (e.g., Bernoulli data), the headroom of active  over random annotation is more limited.\looseness=-1
\end{AIbox}

\section{Estimating optimal policies on real data}
\label{sec:estimating-policies}

The theoretical results in Section~\ref{sec:optimal-policies} derive optimal annotation policies under the assumption that the key distributional parameters governing the relationship between the expensive rater ($H$) and the cheap rater ($G$) are known. In reality, these parameters --- namely $\mathrm{Var}(H)$, $\mathrm{MSE}(H, G) = \mathbb{E}[(H-G)^2]$, and $U = \mathbb{E}[(H-G)^2 \mid X]$  --- must be estimated (imperfectly). Furthermore, the optimal threshold $\tau^*$ and scaling factor $\gamma(\tau^*)$ for the active policy $\pi_\mathrm{active}$ in Proposition~\ref{prop:optimal-budget} also depend on conditional versions of these unknown quantities (e.g., the conditional MSE, $\mathbb{E}[(H - G)^2 \mid U \leq \tau]$).

Some of these estimates can be derived automatically from the model itself, for example if $g(x) \in [0,1]$ is a binary classifier, we may choose $u(x) = g(x)(1-g(x))$, which is equal to $\E[(H-g(x))^2 \mid X = x]$ when $g(x) =  \P(H=1 \mid X = x)$. Alternatively, $u(x)$ can be an entirely separate prediction, such as by asking an LLM directly for its confidence~\cite{kadavath2022language,xiong2024can}. For the parameters $\Var(H)$, $\MSE(H, G)$, $\gamma(\tau^*)$, and $\tau^*$, we explore estimating them using the following approaches:\looseness=-1

\textbf{Policy transfer from related datasets (A1).} In this approach, we \emph{transfer} all parameters necessary for $\pi$ from a separate, but related, dataset.  Specifically, in Section~\ref{sec:datasets}, we use data from the Chatbot Arena dataset~\cite{zheng2023judging, chiang2024chatbot} to estimate the win-rate of GPT-4 over Claude 2.1, but transfer the parameters for $\pi_\mathrm{random}$ and $\pi_\mathrm{active}$ from a separate set of comparisons between different available models. See Section~\ref{sec:datasets}  for details. We also calibrate $G$ using Platt scaling~\cite{platt1999probabilistic} on the transfer dataset.\looseness=-1

\textbf{Policy burn-in on the first $n_b$ examples (A2).} When a suitable transfer dataset is not available as in A1, we can take a hybrid approach where we start by sampling $H$ for the first $n_\mathrm{b} = 200$ examples with probability $1$, and then use them to estimate the parameters necessary for $\pi_\mathrm{active}$ and $\pi_\mathrm{random}$.  We also calibrate $G$ using Platt scaling on these $n_b$ examples. As a fair comparison to the baseline method of only using $H$, we also allow these $n_b$ samples to be used as additional data for estimating $\theta = \E[H]$. Specifically, we use the (estimated) inverse-variance-weighted average of the annotation policy $\pi$'s estimate, $\hat{\theta}_T^\pi$, and the  classical estimate on the burn-in data, $\hat{\theta}_{n_b} = \frac{1}{n_b}\sum_{i=1}^{n_b} H_i$,
    \begin{equation}
    \label{eq:combo}
        \hat{\theta}^\pi_{T^\mathrm{stop} + n_b}  = \frac{\widehat{\Var}\big(\hat{\theta}_{T^\mathrm{stop}}^\pi \big)}{\widehat{\Var}(\hat{\theta}_{n_b}) + \widehat{\Var}\big(\hat{\theta}_{T^\mathrm{stop}}^\pi\big)} \hat{\theta}_{n_b} + \frac{\widehat{\Var}(\hat{\theta}_{n_b})}{\widehat{\Var}(\hat{\theta}_{n_b}) + \widehat{\Var}(\hat{\theta}_{T^{\mathrm{stop}}}^\pi)} \hat{\theta}_{T^{\mathrm{stop}}}^\pi,
    \end{equation}
    where $\widehat{\Var}(\cdot)$ is also estimated on the burn-in data. Note that $\hat{\theta}_{T^\mathrm{stop} + n_b}$ is still unbiased. 

To get a sense of how close to optimal the estimated policies are,  we also compute an  \textbf{Oracle}: $\pi_\mathrm{active}$ with parameters computed using the whole dataset, and $u(x)$ taken directly as $\vert h(x) - g(x)\vert^2$.\looseness=-1

\subsection{Metrics}
We compare the baseline method $\pi_\mathrm{base}$ of always sampling $H$ with the random policy $\pi_\mathrm{random}$ and the active policy $\pi_\mathrm{active}$. For each policy, we compute the \textbf{mean squared error}, $\E[(\hat{\theta}_T^\pi - \theta^*)^2]$, for a range of budgets $B$ ($\costh$ is normalized to be one "cost unit"), with 95\% bootstrap CIs shown over $2k$ trials. We then compute the \textbf{mean effective budget}, which we define as the budget $B'$ required for $\pi_\mathrm{base}$ to achieve the same MSE as the given policy $\pi$ at a budget $B$. If $\pi$ is more cost-effective than $\pi_\mathrm{base}$, then $B'$ will be larger than $B$ (higher is better). Finally, we also compute the \textbf{mean cost savings} for a given mean-squared error, which we define as the budget deficit relative to $\pi_\mathrm{base}$ required to achieve that target error (higher is better). By definition, we have that the mean effective budget for $\pi_\mathrm{base}$ is the line $y = x$ (since $B' = B$ always), while  the  cost savings for $\pi_\mathrm{base}$ is  $0$.

\subsection{Datasets}
\label{sec:datasets}

We report experimental results on four datasets, which span a diverse range of weak and strong raters $G$ and $H$ and distributional characteristics. For each task, we calculate the target metric $\theta^* = \E[H]$ using the full dataset. For simplicity we assume that the total number of  data points $X_t$  is at least $\lceil B / \costg \rceil$, and sample with replacement from the original dataset until this condition is met. We leave treatment of finite datasets where $T^\mathrm{stop} \leq T^\mathrm{max}$ (and the constraint is active) to future work. See Appendix~\ref{app:additional_results} for results on two additional datasets, ImageNet~\cite{deng2009imagenet} and Seahorse~\cite{clark-etal-2023-seahorse}.

\textbf{Chatbot Arena.} The Chatbot Arena dataset~\cite{zheng2023judging, chiang2024chatbot} evaluates LLMs  via pairwise comparisons (i.e., eliciting preferences for response A vs. B from two models for the same query). Among the $64$ models present in the $57k$ total comparisons in the dataset, we focus on estimating the win-rate of GPT-4 (specifically, the $11/06$ preview model) versus Claude 2.1---as they are both strong models, and  also have the most pairwise comparisons in the dataset ($1073$ total), which allows us to get a reliable estimate of $\theta^*$. We  model $H$ via the majority preference from 10 Gemini 1.5 Flash~\cite{geminiteam2024gemini15unlockingmultimodal} evaluations (5 samples each comparing A vs. B and B vs. A to mitigate position bias).
$G$ is the win probability predicted by a Gemma-3 4B model~\cite{gemmateam2025gemma3technicalreport} which has been fine-tuned on the other model comparisons from the dataset to predict the Gemini labels.  $U$ is computed as $G(1-G)$.

\textbf{Chatbot Arena (estimated easy/hard split).} In an effort to include a dataset with more (identifiable) heteroskedasticy, we also include a filtered version of the GPT-4 versus Claude 2.1 task described above, where we  construct a dataset slice containing only the examples corresponding to the bottom 25\% and top 25\% of Gemma's uncertainty estimates (we use $U$ as the  metric).  While partly manipulated, this scenario is designed to test for potential gains from actively choosing when to query the expensive rater, as per the intuition from Section~\ref{sec:synthetic_data}, where it was shown how higher $\mathrm{Var}(U)$ benefits active policies (though note this may not be true if the estimated $U$ is inaccurate).

\begin{figure}[!t]
    \centering
    \includegraphics[width=1\linewidth]{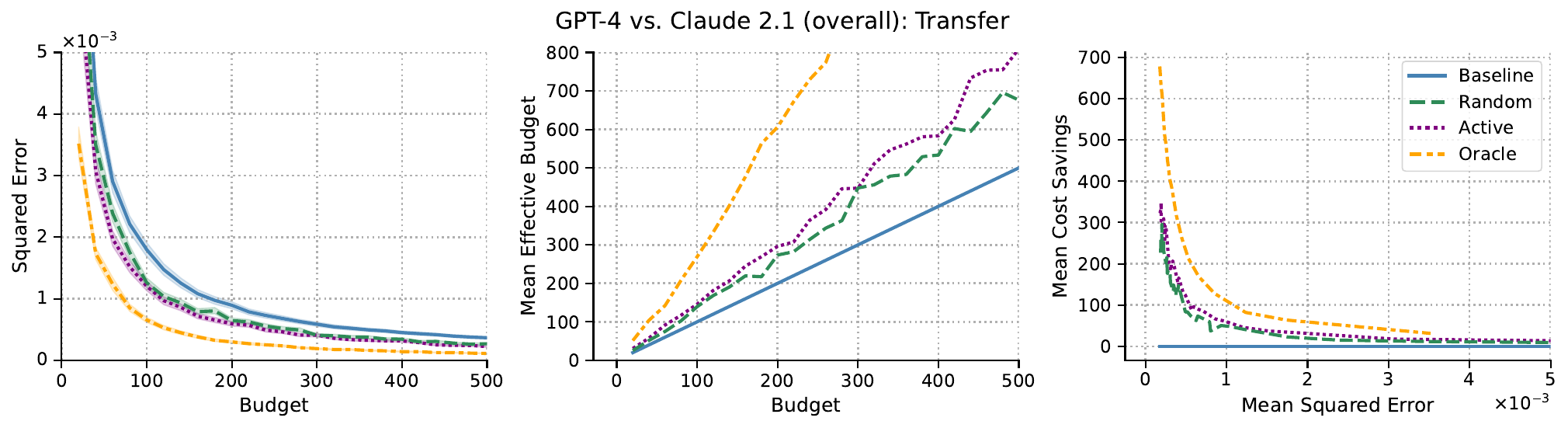}
    \includegraphics[width=1\linewidth]{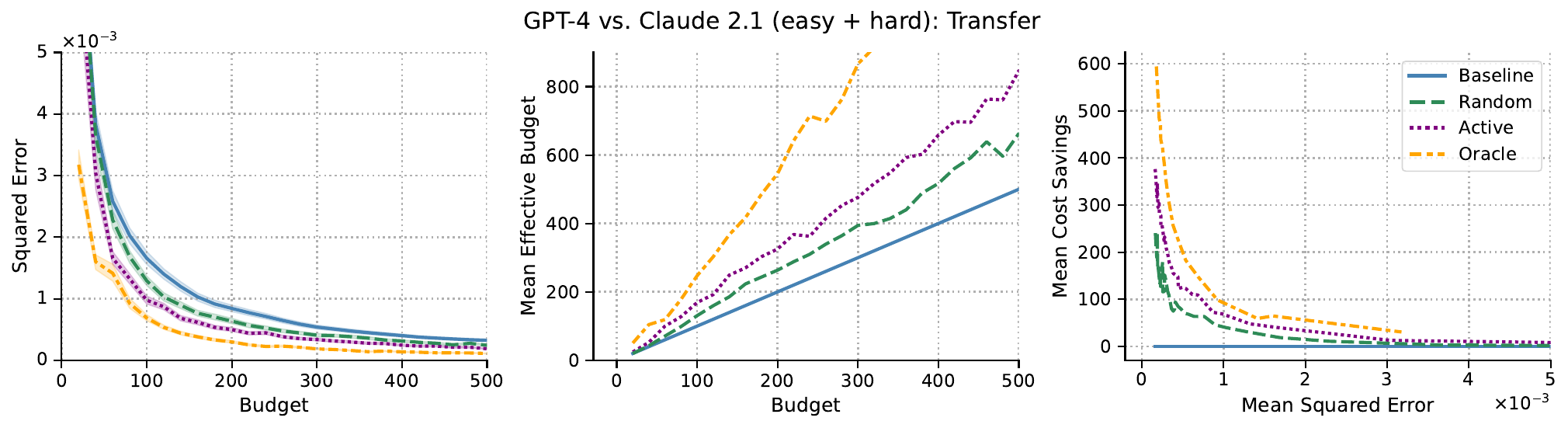}
    \vspace{-12pt}
    \caption{Results for estimating the win-rate of GPT-4 vs. Claude-2.1 on the Chatbot Arena dataset when  using policy transfer (see approach A1 in Section~\ref{sec:estimating-policies}). Both $\pi_\mathrm{random}$ and $\pi_\mathrm{active}$ substantially improve estimation quality over $\pi_\mathrm{base}$ for a given budget. Consistent with our theory, $\pi_\mathrm{active}$'s performance benefits are substantially magnified on the heterogenous easy/hard split (bottom row). %\jacob{how are errorbars computed?}\adam{see 4.1}
    } 
    \vspace{-10pt}
    \label{fig:arena}
\end{figure}

\textbf{AQA.}  Attributed Question Answering (AQA)~\cite{bohnet2023attributed} assesses if a QA system's answer is both correct \emph{and} supported by the text of a document provided as evidence for it (also by the QA system). We evaluate the highest-scoring "retrieve-and-read" system from the dataset. $H$ is a binary human label that is 1 if the answer is both correct \emph{and} attributable, and 0 otherwise. $G$ is the probability predicted by an 11B parameter T5 model~\cite{JMLR:v21:20-074} that the answer is attributable. The T5 model is finetuned on a collection of natural language entailment tasks~\cite{honovich2022true}. $U$ is computed as $G(1-G)$.\looseness=-1

\subsection{Results}
\label{sec:experimental-results}

 \begin{figure}[!t]
    \centering
    \includegraphics[width=\linewidth]{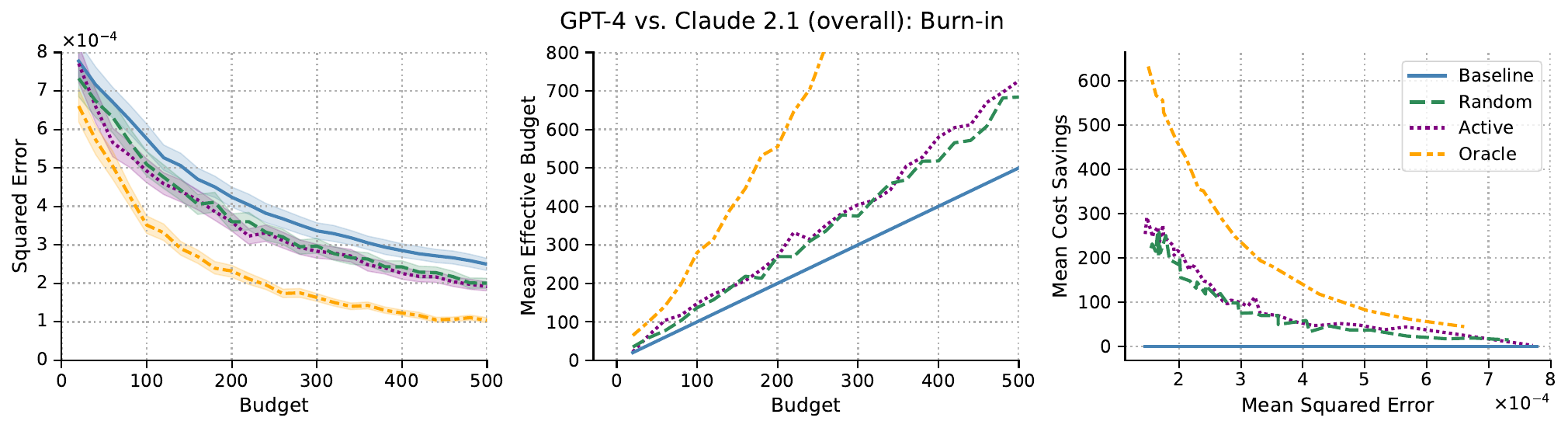} 
    \includegraphics[width=\linewidth]{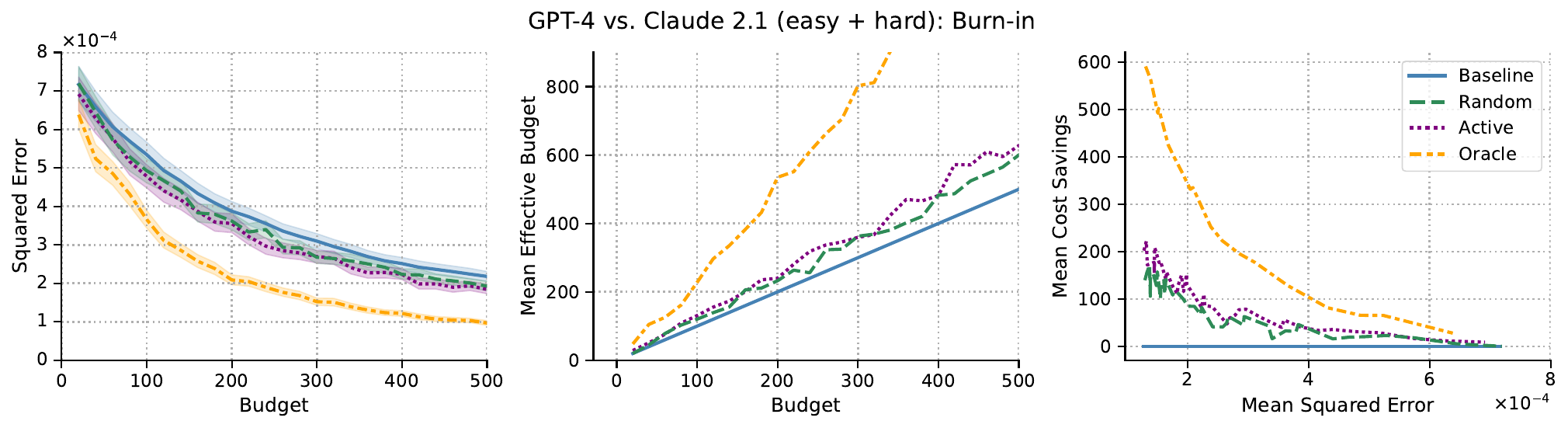} 
    \includegraphics[width=\linewidth]{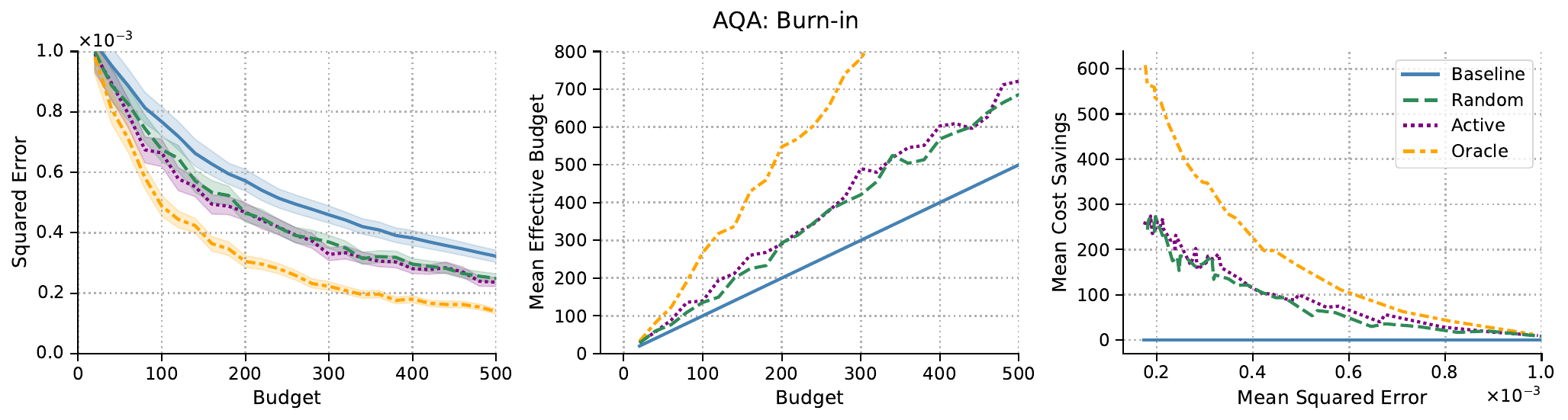} 
    \vspace{-14pt}
    \caption[]{Results on Chatbot Arena and AQA using $200$ examples as a burn-in to estimate  policy parameters, and then switching to the initialized annotation policy thereafter  (see approach A2 in Section~\ref{sec:estimating-policies}; note that budgets $B$ on the x-axis  reflect the ``additional'' budget used \emph{after} the burn-in examples). While the absolute differences in squared errors for the estimated means are smaller than in the transfer setting in Figure~\ref{fig:arena}, both $\pi_\mathrm{random}$ and $\pi_\mathrm{active}$ still achieve consistent improvements over $\pi_\mathrm{base}$.\looseness=-1} %in  error reduction,  mean effective budget, and mean cost savings over $\pi_\mathrm{base}$.}
    \label{fig:burnin}
    \vspace{-10pt}
\end{figure}

Figure~\ref{fig:arena} shows results for the Chatbot Arena datasets using the \emph{transfer} approach (A1), while Figure~\ref{fig:burnin} shows results for all datasets using the \emph{burn-in} approach (A2).\footnote{We also apply power tuning~\cite{angelopoulos2023ppi++} after all samples are collected. See Appendix~\ref{sec:power_tuning} for details.} As expected, the absolute improvement for both $\pi_\mathrm{active}$ over $\pi_\mathrm{random}$ and $\pi_\mathrm{random}$ over $\pi_\mathrm{base}$ is greatest in the transfer setting in Figure~\ref{fig:arena}, where the parameters of $\pi_\mathrm{random}$ and $\pi_\mathrm{active}$ can be approximated in advance. In particular, to achieve a root mean-squared error (RMSE) of $0.05$, $\pi_\mathrm{active}$ requires only $\approx 40\%$ of the budget required by $\pi_\mathrm{base}$ in the overall setting of Chatbot Arena, and only $\approx 50\%$ of the budget in the easy/hard setting. These cost savings become even more pronounced the more precise (i.e., lower MSE) estimates are required to be. 
In Figure~\ref{fig:burnin}, where the first  $n_b = 200$ examples are fully labeled in order to estimate the parameters of $\pi_\mathrm{random}$ and $\pi_\mathrm{active}$, the absolute \emph{difference} in MSE is smaller for $\pi_\mathrm{random}$ and  $\pi_{active}$ over $\pi_\mathrm{base}$,  though the subsequent cost savings over $\pi_\mathrm{base}$ for achieving lower and lower MSE (that is, past the MSE of the initial $n_b$ sample estimate) are  consistent. 

That said, while the annotation policies $\pi_\mathrm{random}$ and $\pi_\mathrm{active}$ help achieve more accurate estimates of $\E[H]$ overall, the results in Figure~\ref{fig:arena} and Figure~\ref{fig:burnin} also show that the extent of the improvement in estimation accuracy varies per dataset (see also the additional results in Appendix~\ref{app:additional_results}). In particular, the best results are obtained on the easy/hard split of the Chatbot Arena dataset, where (i) the weak annotator $G$ is a good proxy of the strong annotator $H$ (both are LLMs), and (ii) there is more variability in the difficulty of examples according to the predicted $U$, resulting in a greater opportunity  for improvement for $\pi_\mathrm{active}$.  On the other hand, while results on AQA and the homogeneous split of the Chatbot Arena dataset also show improvements for $\pi_\mathrm{random}$ and $\pi_\mathrm{active}$ over $\pi_\mathrm{base}$,  the relative improvement of $\pi_\mathrm{active}$ over $\pi_\mathrm{random}$ is fairly small---indicating that while the weak annotator $G$ that is used  is relatively good on average, there is not much variability in its estimated uncertainty, $u(x)$, on those distributions. To that point, when we compare to performance using the oracle active policy, it is also clear that the estimated $u(x)$ is also far from perfect. Even on the datasets where the improvement due to the estimated active policy is small, the oracle policy which has knowledge of the true error of $G$ often promises significant headroom: indicating that the working on better autorater uncertainty estimation is a promising and important direction for future work.\looseness=-1

\section{Conclusion}
\label{sec:conclusion}

This paper introduces \textbf{active evals}, a framework that strategically combines cheap, automated raters with more expensive, accurate alternatives to improve evaluation efficiency. We derive annotation policies that are optimal in the sense of minimizing expected error under annotation budget constraints, and we empirically characterize the conditions under which such policies yield improvements over non-hybrid (e.g., human-only) and non-active hybrid alternatives. However, the theoretically-optimal annotation policies developed here are distribution dependent, with a few task-specific parameters that must be estimated. Furthermore, truly optimal \emph{active} annotation depends on having an accurate \emph{conditional} uncertainty estimates, which may be more difficult to obtain than the original expected quality estimates, particularly in a "cold-start" setting. However, many realistic evaluation scenarios involve incrementally adding new models to existing benchmarks; as shown in \S\ref{sec:experimental-results}, policy transfer can work quite well. 
All that said, the optimal fixed sampling rate policy that we derive in this paper, always provides a substantial improvement over baselines in our experiments.
Our results indicate that even when active sampling is difficult for the reasons outlined above, the simple strategy remains useful.\looseness=-1

\bibliographystyle{plainnat}
\bibliography{bib}
\clearpage
\appendix
\addcontentsline{toc}{chapter}{Appendices} % For main TOC
\etocsetnexttocdepth{2}
\localtableofcontents
\clearpage

\section{Broader impacts}
\label{sec:impacts}
This paper describes fundamental research on the evaluation of generative AI systems, which is a core technical challenge.  Hybrid active evaluation has the potential to improve the cost/accuracy tradeoff of system evaluation, which can make high-quality AI systems easier to build, deploy, and monitor. We do not speculate about broader impacts that may follow from this technical contribution.
\section{Additional theoretical results}
\label{app:additional-theory}

\subsection{Derivation of $\Error_T(\pi)$}
\label{sec:derivation}
We provide a short derivation of $ \Error_T(\pi)$ in \eqref{eq:error_t}. Because the estimator $\hat\theta_T^{\pi}$ is unbiased, 
\begin{align}
\E\left[\left(\hat\theta_T^\pi - \theta^*\right)^2\right] =  \Var(\hat\theta_T^{\pi}) = \frac{1}{T} \Var(\Delta^\pi)
\end{align}
when $\pi$ and $g$ are fixed, and where $\Delta^\pi = G + (H - G)^2 \frac{\xi}{\pi(X)}$. Then,
\begin{align}
\Var(\Delta^\pi)
 &= \E\left[\left(G + (H - G)\frac{\xi}{\pi(X)}\right)^2\right] - (\theta^*)^2 \\
    & = \E\left[ G^2 \right] + \E\left[\left((H - G)\frac{\xi}{\pi(X)}\right)^2\right] + 2\E\left[G(H - G)\frac{\xi}{\pi(X)}\right] - (\theta^*)^2 \\
    & = \E\left[ G^2 \right] + \E\left[(H - G)^2\frac{1}{\pi(X)}\right] + 2\E\left[G(H - G)\right] - (\theta^*)^2 \\
    & = \Var(H) - \E[(H-G)^2] + \E\left[(H - G)^2\frac{1}{\pi(X)}\right].
\end{align}

\subsection{Power tuning}
\label{sec:power_tuning}
\citet{angelopoulos2023ppi++} proposed "power tuning" as a way to improve upon the standard PPI estimator by allowing the estimator to adapt to the "usefulness" of the supplementary predictions (here, the weak rater G) with a tuning parameter $\lambda \in \mathbb{R}$. We now extend this to our setting. 

Let us consider a modified version of our estimator, with some fixed policy $\pi$ and $\lambda \in \mathbb{R}$:
\begin{equation}
    \hat\theta_T^{\lambda} = \frac{1}{T}\sum
    \limits_{t=1}^T \lambda G_t + (H_t - \lambda G_t)\frac{\xi_t}{\pi(X_t)}.
\end{equation}
For all values of $\lambda$, this estimator is unbiased.
Our job is to pick the value with minimum error. Following the previous derivation in Section \ref{sec:derivation}, the error of the estimator is
\begin{align}
\Error_{T, \pi}(\lambda) = \frac{1}{T} \left(\Var(H) - \E[(H-\lambda  G)^2] + \E\left[(H - \lambda G)^2\frac{1}{\pi(X)}\right] \right),
\end{align}
which is optimized by
\begin{align}
    \lambda^* &= \argmin_{\lambda \in \mathbb{R}} \E\left[(H - \lambda G)^2\left(\frac{1}{\pi(X)} - 1\right)\right] \\
    &=\argmin_{\lambda \in \mathbb{R}} \lambda^2 \E\left[G^2 \left(\frac{1}{\pi(X)} - 1\right) \right] - 2\lambda \E \left[HG \left(\frac{1}{\pi(X)} - 1\right) \right].
\end{align}

The above expression is quadratic in $\lambda$, and its optimizer is
\begin{equation}
    \lambda^* = \frac{\E\left[H G\left(\frac{1}{\pi(X)}-1\right)\right]}{\E\left[G^2\left(\frac{1}{\pi(X)}-1\right)\right]},
\end{equation}
which can be estimated in any consistent way, e.g., by its prediction-powered plug-in that can be computed after sampling all $(X_t, G_t, H_t, \xi_t)$ as:
\begin{equation}
    \hat\lambda_T = \frac{\frac{1}{T}\sum\limits_{t=1}^T \left(G_t^2 + (H_tG_t - G_t^2)\frac{\xi_t}{\pi_t(X_t)}\right)\left(\frac{1}{\pi_t(X_t)}-1\right)}{\frac{1}{T}\sum\limits_{t=1}^T G_t^2\left(\frac{1}{\pi_t(X_t)}-1\right)}.
\end{equation}

\subsection{Optimal random annotation: discrete time case }
The following proposition is the full version of Proposition~\ref{prop:optimal-random}---with the constraint that $T^{\rm stop}$ is an integer.
This leads to a substantially more complex optimization problem; we show the solution here, but we do not implement it in practice.
\begin{proposition}
    \label{prop:optimal-random-integer}
    Let $(X_1, G_1, H_1), \ldots, (X_T, G_T, H_T)$, $T \in \N$, be an i.i.d. sequence of real-valued random variables with joint distribution $P$, and define $\Error$, $\Cost$, and $\Pi^{\rm random}$ as above.
    Additionally, define the optimization problem
    \begin{equation}
        \begin{aligned}
            \minimize_{\substack{\pi \in \Pi^{\rm random}\\T^{\rm stop} \in \N_+}} \quad & \Error_{T^{\rm stop}}(\pi) \\
            \st \quad & \Cost_{T^{\rm stop}}(\pi) \leq B.
        \end{aligned}
        \label{problem:optimal-random-integer}
    \end{equation}
    Then the optimal solution to Problem~\eqref{problem:optimal-random-integer} is either $\pi^*(x) = 1$ or
    \begin{equation}
        \pi^*(x) = \frac{B-k^*c_g}{k^*c_h}.
    \end{equation}
    for all $x \in \cX$, where
    \begin{equation}
        k^* = \argmin_{k \in \mathcal{K}}\frac{1}{k}\left(\Var(H)-\E[(H-G)^2]\right) + \frac{c_h}{B - kc_g}\E[(H-G)^2],
    \end{equation}
    and
    \begin{equation}
        \mathcal{K} = \left\{  \left\lfloor B\frac{1 + \sqrt{ \frac{c_h}{c_g}\frac{\E[(H-G)^2]}{\Var(H)-\E[(H-G)^2]}}}{c_g - c_h\frac{\E[(H-G)^2]}{\Var(H)-\E[(H-G)^2]}}\right\rfloor ,  \left\lceil B\frac{1 + \sqrt{ \frac{c_h}{c_g}\frac{\E[(H-G)^2]}{\Var(H)-\E[(H-G)^2]}}}{c_g - c_h\frac{\E[(H-G)^2]}{\Var(H)-\E[(H-G)^2]}}\right\rceil\right\}.
    \end{equation}
\end{proposition}
It is easy to disambiguate between $p^*=1$ and the optimal policy based on $k^*$ by comparing the objective values directly.

\subsection{Extension to convex M-estimators}
Here we give an extension of Proposition~\ref{prop:optimal-budget} to general convex M-estimators~\cite{van2000asymptotic}.
Consider a convex loss function, $\ell_\theta$ for some $\theta \in \R^d$, equipped with the simplified notation $\ell_{\theta, t} = \ell_{\theta}(X_t,H_t)$ for all $t \in \N$ and $\ell_{\theta, t}^g = \ell_{\theta}(X_t,G_T)$.
We also use $\ell_\theta = \ell_{\theta}(X,H)$ and $\ell_\theta^g = \ell_\theta(X,G)$ for generic points $(X,G,H)\sim P$.
The target of estimation is the population minimizer, $\theta^* = \argmin_{\theta \in \R^d}\E[\ell_{\theta}]$.
The active estimator is 
\begin{equation}
    \hat\theta_T = \argmin_{\theta \in \R^d} \frac{1}{T}\sum\limits_{t=1}^T \Delta_{\theta, t} \quad \text{ where } \Delta_{\theta, t} = \ell_{\theta,t}^g + \left(\ell_{\theta,t} - \ell_{\theta,t}^g\right) \frac{\xi_t}{\pi_t(X_t)},
\end{equation}
for some sequence of annotation policies $\pi_t$, $t \in \N$.
For the purpose of deriving optimal annotation policies when $\pi_t$ is fixed as in Section~\ref{sec:optimal-policies}, we will also define 
\begin{equation}
    \hat\theta_T^\pi = \argmin_{\theta \in \R^d} \frac{1}{T}\sum\limits_{t=1}^T \Delta_{\theta, t} \quad \text{ where } \Delta_{\theta, t} = \ell_{\theta,t}^g + \left(\ell_{\theta,t} - \ell_{\theta,t}^g\right) \frac{\xi_t}{\pi(X_t)}.
\end{equation}

Unlike the estimator in the case of mean estimation from Section~\ref{sec:optimal-policies}, $\hat\theta_T^\pi$ does not have a closed-form variance in finite samples.
The standard solution in the analysis of M-estimators is to appeal to the asymptotic linearity of M-estimators to analyze the variance~\cite{van2000asymptotic}, as is done in Theorem 1 of ~\citet{zrnic2024active}.
The result below combines the aforementioned theorem with standard parametric analysis to give the asymptotic distribution of the squared error.
\begin{proposition}
    \label{prop:m-estimator-mse-dist}
    Let $\ell_\theta$ be smooth (see Assumption 1 in~\cite{zrnic2024active}) and define the Hessian $W_{\theta^*} = \nabla^2\E[\ell_{\theta^*,t}]$.
    Then if $\hat\theta^{\pi}_T \overset{p}{\to} \theta^*$, we have
    \begin{equation}
        \sqrt{T}(\hat\theta_T^{\pi} - \theta^*) \overset{d}{\to} \cN(0,\Sigma^*), 
    \end{equation}
    where $\Sigma^* = W_{\theta^*}^{-1} \Var\left( \nabla \ell_{\theta^*,t}^g + \left( \nabla \ell_{\theta^*,t} - \nabla \ell_{\theta^*,t}^g \right) \frac{\xi_t}{\pi(X_t)} \right) W_{\theta^*}^{-1}$.
    Therefore, we have
    \begin{equation}
        T\left\|\hat\theta_T^{\pi} - \theta^*\right\|^2_2 \overset{d}{\to} \sum_{j \in [d]} \lambda_j \zeta_j,
    \end{equation}
    where $\zeta_j \iidsim \chi^2_1$ for all $j \in [d]$ and $\lambda_j$ is the $j$th eigenvalue of $\Sigma^*$.
\end{proposition}
The above proposition gives us consistency of the active estimator, and more importantly, the asymptotic distribution of the squared error.
Since $\E[\zeta_j] = 1$ for all $j$, and the sum of the eigenvalues of a square matrix is equal to the trace, we know the mean-squared error is asymptotically equal to $\Error_T(\pi) = \frac{1}{T}\Tr(\Sigma^*)$.
With this in hand, we can use the same strategy from earlier to find the optimal annotation policy, using the asymptotic approximation of the error. For simplicity, here we assume that we are always on the interior of the constrained optimization problem, i.e., we solve for unconstrained $\pi(x)$ while assuming that $\gamma^* \sqrt{u(X)} \leq 1$. That said, a more rigorous treatment analogous to that in Proposition~\ref{prop:optimal-budget} can also be applied here, which we leave to future work.
\begin{proposition}
    \label{prop:convex-budget}
    In the setting of Proposition~\ref{prop:m-estimator-mse-dist}, let $(X_1, G_1, H_1), \ldots, (X_T, G_T, H_T)$, $T \in \N$, be an i.i.d. sequence of real-valued random variables with joint distribution $P$, and define $\Error_T(\pi) = \frac{1}{T}\sum_{j\in [d]}\Tr(\Sigma^*)$.
    Furthermore, define $\Cost$ and $\Pi$ as in Proposition~\ref{prop:optimal-budget}.
    
    Construct the optimization problem
    \begin{equation}
        \begin{aligned}
            \minimize_{\pi \in \mathcal{F},~T^{\rm stop} \in \R_{>0}} \quad & \Error_{T^{\rm stop}}(\pi) \\
            \st \quad & \Cost_{T^{\rm stop}}(\pi) \leq B.
        \end{aligned}
        \label{problem:convex-budget}
    \end{equation}
    where $\mathcal{F} = \{ x \mapsto f(x) : f(x) \in (0, \infty); \forall x \in \cX \}$.
    Then the solution to Problem~\eqref{problem:convex-budget} is
    \begin{equation}
        \pi^*(x) = \sqrt{\frac{c_g}{c_h} \cdot \frac{u(x)}{C}}
    \end{equation}
    where
    \begin{equation}
         u(x) = \E\left[\Tr\left(W_{\theta^*}^{-1}\left( \nabla \ell_{\theta^*} - \nabla \ell_{\theta^*}^g \right)\left(\nabla \ell_{\theta^*} - \nabla \ell_{\theta^*}^g \right)^\top W_{\theta^*}^{-1}\right) \mid X = x\right],
    \end{equation}
    and 
    \begin{equation}
    \begin{split}
         C = \Tr\Big(W_{\theta^*}^{-1}\Big(\E\left[\nabla \ell_{\theta^*}^g (\nabla \ell_{\theta^*})^\top + (\nabla \ell_{\theta^*} - \nabla \ell_{\theta^*}^g)(\nabla \ell_{\theta^*}^g)^\top \right] - \E[\nabla \ell_{\theta^*}]\E[\nabla \ell_{\theta^*}]^\top\Big) W_{\theta^*}^{-1}\Big).
        \end{split}
    \end{equation}
\end{proposition}
\begin{remark} When $\pi^*(x) \leq 1$, $\forall x \in \cX$, then $\pi^*$ is also optimal for Problem~\eqref{problem:convex-budget} solved for $\pi \in \Pi$.
\end{remark}

\subsubsection{Mean estimation}
In the case of mean estimation, the loss function takes the form
\begin{equation}
    \ell_{\theta}(x,h) = \frac{1}{2}(h - \theta)^2,
\end{equation}
where $\nabla \ell_{\theta^*}(X,H) - \nabla \ell_{\theta^*}(X,G) = H - G$, and $W_{\theta^*}$ is the identity matrix. Plugging back into $\pi^*$ in Proposition~\ref{prop:convex-budget} recovers $\pi_\mathrm{active}$ from Proposition~\ref{prop:optimal-budget} without clipping ($\tau^* = \infty$), i.e.,
$$\sqrt{\frac{\costg}{\costh} \frac{\E[(H - G)^2 \mid X = x]}{\Var(H) - \E[(H - G)^2]}}.$$

\subsubsection{Generalized linear models}
In the case of GLMs, the loss function takes the form
\begin{equation}
    \ell_{\theta}(x,h) = -hx^\top \theta + \psi (x^\top \theta)
\end{equation}
for some convex log-partition function $\psi$.
Thus, $\nabla \ell_{\theta^*}(X,H) - \nabla \ell_{\theta^*}(X,G) = (G - H)X$.
So, again by the linearity of the trace, we have that $$\pi^*(x) \propto \sqrt{\E\left[(H - G)^2 \mid X = x \right] \Tr\left( W_{\theta^*}^{-1}xx^\top W_{\theta^*}^{-1} \right)}.$$

\subsection{Effects of noisy policy parameters on estimator variance}
In practice, we will be using only an imperfect estimate of $u(x)$ for $\pi_\mathrm{active}$, which can negatively affect the performance of $\pi_\mathrm{active}$  to a substantial degree,  as we have seen for some of the datasets in Section~\ref{sec:estimating-policies}. Similarly, we will also only be using imperfect estimates of  the optimal scaling and thresholding parameters used in $\pi_\mathrm{active}$, which further limit performance.

There are two main factors that affect the error of a policy: 
\begin{enumerate}[leftmargin=*]
    \item The variance, $\Var(\Delta^\pi)$, of each active increment $\Delta^\pi$, where $\Delta^\pi = G + (H - G) \frac{\xi}{\pi(X)}$.
    \item The average sample size at which the estimator runs out of budget, $T^\mathrm{stop} = \frac{B}{\costh\E[\pi(X)] + \costg}$.
\end{enumerate}
In this section, we provide some additional theoretical analysis on the first factor, i.e., the increase in $\Var(\Delta^\pi)$ due to the mispecification error of an \emph{estimated} active policy, while noting that the total error will be further affected by the relative increase/decrease of the mean sampling rate, $\E[\pi(X)]$.

\begin{proposition}
    \label{prop:errors-u}
    In the same setting as Proposition~\ref{prop:optimal-budget}, let $\tilde \pi : \cX \to (0,1]$ be any function satisfying 
    \begin{equation}
        \E\left[ \frac{1}{\tilde \pi(X)} - \frac{1}{\pi^*(x)} \right] \leq \delta,
    \end{equation}
    where $\pi^*$ is the oracle estimate of $\pi_\mathrm{active}$. Let $(H-G)^2 \overset{\rm a.s.}{\leq} b$.
    Then $\Var(\Delta^{\tilde{\pi}}) \leq \Var(\Delta^{\pi^*}) + b \delta$.
\end{proposition}

If we simply things by assuming an additive error model for a policy without thresholding (i.e., $\tau^* = \infty$), we can refine the bound somewhat further:

\begin{corollary}
\label{cor:additive_error}
    Let $\tilde\pi = \tilde\gamma \sqrt{\tilde u(x)}$, where $\tilde\gamma = \gamma^* + \delta_{\gamma}$, $\tilde u(x) = u(x) + \delta_{u}(x)$, and $u(X) \overset{\rm a.s.}{\geq} \epsilon$. Further assume that $\tilde\pi$ is admissible, i.e., $\tilde\pi(x) \in (0, 1]$ $\forall x$. Then, up to first-order terms in $\delta_\gamma$ and $\delta_u(x)$,
    \begin{equation}
        \Var(\Delta^{\tilde\pi}) \leq  \Var(\Delta^{\pi^*}) + b\left(\frac{|\delta_{\gamma}|}{(\gamma^*)^2 \sqrt{\epsilon}} + \frac{1}{2\gamma^*\epsilon^{3/2}} \E[|\delta_U(X)|] \right).
    \end{equation}
\end{corollary}

We can make a few observations about the results in Proposition~\ref{prop:errors-u} and Corollary~\ref{cor:additive_error}. First, as long as the error, $(H - G)^2$ is bounded, and the estimated inverse propensity score $1/ \tilde{\pi}(X)$ is not significantly higher than the oracle inverse propensity score $1 / \pi^*(X)$ on average, then the increase in variance over the oracle will not be that large. Generally speaking, this is satisfied when the estimated policy is not \emph{overconfident} on examples that in fact have high error. Of course, regularizing the estimated policy to be underconfident on all examples is also not always a satisfying solution: as $\mathbb{E}[\tilde{\pi}(X)] \rightarrow 1$, we obtain a policy that is no better than $\pi_\mathrm{base}$. Similarly, as seen in Section~\ref{sec:synthetic_data}, the headroom for $\pi_\mathrm{active}$ over $\pi_\mathrm{base}$ is largest when $(H - G)^2$ is \emph{not} bounded (e.g., the Gaussian data setting compared to the Bernoulli data setting), as large $(H - G)^2$ also increase the possible variance of $U$. This reinforces the importance of having \textbf{accurate uncertainty estimates} when computing active policies.

\subsection{Optimal active sampling of   input evaluation queries}
This section shows how to optimally choose the distribution of $X$.
In contrast, Section~\ref{sec:optimal-policies} in the main paper focuses only on querying annotators for $H$ given i.i.d. samples from the fixed distribution $P$ for $X$.
Deciding \emph{which inputs} to sample can be a more difficult problem than deciding whether to annotate a given input sample  because $\cX$ can be large and complex.
However, we can always apply the optimal rules to a coarse stratification of $\cX$. Towards this end, we define the estimator
\begin{equation}
\begin{split}
    \hat\theta^{Q, \pi}_T = \frac{1}{T}\sum\limits_{t=1}^T \Delta_t, \text{ where } \quad&\Delta_t^{Q, \pi} = \frac{d P}{dQ}(X_t)\left(G_t + \left(H_t - G_t\right)\frac{\xi_t}{\pi(X_t)} \right), \\ &X_t \iidsim Q,\; H_t \sim P_{H \mid X},\; G_t \sim P_{G \mid X}
\end{split}
\end{equation}
which is our previous estimator with a fixed policy $\pi$, and where the $X$ are sampled from a distribution $Q$, and the distribution of $H \mid X$ and $G \mid X$ remain unchanged. 
This estimator is unbiased for $\theta^*$, and a straightforward calculation gives that the error of the estimator is
\begin{align}
    \Error_T(Q, \pi) &= 
    \E_Q\left[\left(\hat\theta_T^{Q,\pi} - \theta^*\right)^2\right] \\ 
    &= \frac{1}{T} \Var(\Delta^{Q, \pi}) \\
    &= \frac{1}{T}\left(\E_{P}\left[\frac{d P}{dQ}(X) \left(H^2 + \left(\frac{1}{\pi(X)} - 1\right)(H-G)^2 \right)\right] - (\theta^*)^2\right).
\end{align}

The goal is to pick a distribution $Q$ to minimize the error of the estimator.
The following proposition gives an explicit form for this optimal sampling distribution.
\begin{proposition}
    \label{prop:optimal-sampling}
    Define $\Error_T$ as above, and define the set of all strictly positive densities, $\mathcal{Q} = \{x \mapsto Q(x) : Q(x) \in \R_{>0} \text{ and } Q \in \Delta^{\cX}\}$.
    Furthermore, define the optimization problem
    \begin{equation}
        \begin{aligned}
            \minimize_{Q \in \mathcal{Q}} \quad & \Error_{T}(Q, \pi)
        \end{aligned}
        \label{problem:optimal-sampling}
    \end{equation}
    for a fixed time $T \in \N$.
    Then the solution to Problem~\eqref{problem:optimal-random} is
    \begin{equation}
        \label{eq:optimal-sampling-policy}
        Q^*(x) = \P(X=x)\frac{\sqrt{\nu(x)}}{\E_P[\sqrt{\nu(X)}]},
    \end{equation}
    where
    \begin{equation}
        \nu(x) = \E_{P}\left[\left(H^2 + \left(\frac{1}{\pi(X)} - 1\right)(H-G)^2 \right) ~\Big\vert~ X = x\right]
    \end{equation}
    for all $x \in \cX$.
\end{proposition}

We leave empirical exploration of active input sampling to future work.

\subsection{Informative special cases for $\pi_\mathrm{active}$}
\label{app:active-cases}

Prior work \cite{zrnic2024active, gligoric2024can} target some fixed, prespecified value (i.e., some ratio $n/ N$) for  $\E[\pi(X)]$. A key distinction of this work is that we also optimize $\E[\pi(X)]$, which will depend strongly on $\costg/\costh$, that is, the cost ratio of $G$ to $H$. In this section we analyze two extreme, but informative cases, for active sampling when either $\costg/\costh = 0$ or $\costg/\costh = \infty$, that serve to illustrate how $\E[\pi_\mathrm{active}(X)]$ for the cost-optimal policy $\pi_\mathrm{active}$ can consequently be as extreme as $0$ or $1$.

\subsubsection*{Optimal policy for $c_g = 0$}

We start with the special case where $c_g = 0$, so that we can obtain essentially infinitely many queries of the weak rater $G$ irrespective of the budget constraint. In this case, we expect that unless $G$ has a prohibitively large error $\E[(H-G)^2]$, we can purely rely on querying $G$, and overcome its error with sufficiently many samples. Indeed, let us assume that $\E[(H-G)^2] = \E[U] < \Var(H)$. Then we note that for any $\tau > 0$:
\begin{align*}
    \tau\sqrt{\frac{\costg/\costh + \P(U > \tau^2)}{\Var(H) - \E[U\ind{U \leq \tau^2}]}} &= \sqrt{\frac{\tau^2 \P(U > \tau^2)}{\Var(H) - \E[U\ind{U \leq \tau^2}]}}\\
    &\leq \sqrt{\frac{\tau^2 \P(U > \tau^2)}{\E[U] - \E[U\ind{U \leq \tau^2}]}}\\
    &= \sqrt{\frac{\tau^2 \P(U > \tau^2)}{\E[U\ind{U > \tau^2}]}} \leq 1,
\end{align*}
where the first inequality is due to our assumption that $\E[U] < \Var(H)$, and the last inequality follows from $\E[U\ind{U > \tau^2}] > \tau^2 \mathbb{P}(U > \tau^2)$. Consequently, we get that in this case, $$\gamma^*(\tau) = \sqrt{\frac{\P(U > \tau^2)}{\Var(H) - \E[U\ind{U \leq \tau^2}]}}.$$

Suppose for now that we only consider the values $\tau$ where $U$ further satisfies that $\sqrt{U}\gamma^*(\tau) \leq 1$ almost surely, and denote this set by $\mathcal{T}$. Let $\Delta = \Var(H) - \E[U] > 0$. Then we see that 
\begin{align*}
    &\min_{\tau\in\mathcal{T}} c_h \E[\pi_\mathrm{clip}(x; \tau)]\left[\Delta + \E\left[\frac{U}{\pi_\mathrm{clip}(x; \tau)}\right]\right]\\
    =& \min_{\tau\in\mathcal{T}} c_h \E[\sqrt{U}\gamma^*(\tau)]\left[\Delta + \E\left[\frac{U}{\sqrt{U}\gamma^*(\tau)}\right]\right]\\
    =& \min_{\tau\in\mathcal{T}} c_h \E[\sqrt{U}\gamma^*(\tau)]\Delta + c_h \E[\sqrt{U}\gamma^*(\tau)]\E\left[\frac{\sqrt{U}}{\gamma^*(\tau)}\right].
\end{align*}
Since $\gamma^*(\tau)$ is deterministic, it cancels from the second term above, and we get that the annotation cost over $\tau \in \mathcal{T}$ is monotonically increasing in $\gamma^*(\tau)$, meaning that we choose $\tau \to \infty$, which yields $\gamma^*(\tau) \to 0$ (since $P(U > \infty) = 0$). We also note that whenever $\sqrt{U} \leq B$, all $\tau$ such that $\gamma^*(\tau) < 1/B$ are in $\mathcal{T}$ trivially, since this satisfies $\sqrt{U}\gamma^*(\tau) < 1$. In particular, this includes our choice of $\tau \to \infty$, which ensures that $\gamma^*(\tau) \to 0$. Finally, we note that from the proof of Proposition~\ref{prop:optimal-budget} (specifically, Equation~\ref{eq:p_x_lambda}), we have that $\pi(x) = \gamma \sqrt{u(x)}$ minimizes the objective 
\[
(c_h E[\pi(X)] + c_g)\left[\Var(H) - \E[U] + \E\left[\frac{U}{\pi(X)}\right]\right],
\]
over all mappings $\pi \in \{x \mapsto f(x) : f(x) \in (0,\infty); \forall x \in \cX\} $. Since we find that our optimal choice without imposing the constraint $\pi(x; \tau) \leq 1$ is already feasible, it is also optimal for the constrained problem, $\pi \in \{x \mapsto f(x) : f(x) \in (0,1]; \forall x \in \cX\} $. 

\subsubsection*{Optimal policy for $c_h = 0$}

The other extreme case is simpler. When $c_h = 0$, the objective for $\tau^*$ becomes monotonically decreasing in $\pi(x; \tau)$. If we assume that $\tau$ is such that $\gamma^*(\tau) < 1/\tau$, then we find that the expression $$\sqrt{\frac{{c_g}/{c_h} + \mathbb{P}\left(U > \tau^2\right)}{\big(\mathrm{Var}(H) - \mathbb{E}[U \ind{U \leq \tau^2}]\big)_+}}$$ becomes infinite due to $c_h = 0$, and hence we must have $\gamma^*(\tau) = 1/\tau$. However, for any $x$ such that $\pi(x; \tau) < 1$, we have $1/\pi(x; \tau) = \tau/\sqrt{u(x)}$. Consequently, minimizing over $\tau$ results in $\tau = 0$. But this gives $\pi(x;\tau) = \infty$, so that we must have $\pi(x; \tau) = 1$ for all $x$. Intuitively, this makes sense because any $\pi(x; \tau) < 1$ results in an estimator with variance strictly greater than $\Var(H)$, but having $\pi(x; \tau) \equiv 1$ allows us to attain the smallest possible variance of $\Var(H)$. Since there is no effect of these choices on the estimation cost, we choose the lowest variance estimator in this case, and direct all our queries to the strong rater.

\section{Proofs}
\label{app:proofs}

\subsection{Proof of Proposition~\ref{prop:optimal-random}}
\begin{proof}
    Since $\Error_T(\pi)$ is monotone in $T$ for all $\pi$, we should first set $T^{\rm stop}$ to be the largest $T$ for which the constraint holds.
    This value is
    \begin{equation}
        T^{\rm stop} = \frac{B}{c_h p + c_g}.
    \end{equation}
    Plugging this into the objective yields
    \begin{equation}
        (c_h p + c_g)\left( \Var(H) - \E[(H-G)^2] + \frac{1}{p}\E\left[(H-G)^2\right]\right), 
    \end{equation}
    which, after removing terms that do not depend on $p$, is equivalent to minimizing
    \begin{equation}
        c_h p\left( \Var(H) - \E[(H-G)^2] \right) + \frac{c_g}{p}\E\left[(H-G)^2\right]
    \end{equation}
    subject to the constraint that $p \in [0,1]$.

    This is a convex problem in $p$, and we know that the solution lies either on the boundary or on the interior. 
    We will compare the values of the objectives in three cases: $p^* = 0$, $p^* = 1$, and $p^* \in (0,1)$.
    It is clear that $p^* = 0$ is infeasible (unless $H \aseq G$, which renders the problem trivial) because the factor $c_g/p$ appears in the above objective.
    In the case that $p^* = 1$, the objective value is
    \begin{equation}
        c_h \left( \Var(H) - \E[(H-G)^2] \right) + c_g\E\left[(H-G)^2\right].
    \end{equation}
    In the case that $p^* \in (0,1)$, it must be a critical value, so it satisfies the first-order condition
    \begin{equation}
        c_h \left( \Var(H) - \E[(H-G)^2] \right) = \frac{c_g}{(p^*)^2}\E[(H-G)^2],
    \end{equation}
    and thus,
    \begin{equation}
        \label{eq:pstar-squared-interior}
        (p^*)^2 =  \frac{c_g\E[(H-G)^2]}{c_h(\Var(H) - \E[(H-G)^2])}.
    \end{equation}
    However, because we are in the case $p^* \in (0,1)$, the right-hand side above must be positive (otherwise the square root would be imaginary), and it cannot be greater than $1$ (otherwise we would have $p^*>1$, which is a contradiction).
    This gives us that
    \begin{equation}
        p^* \in (0,1) \implies \E[(H-G)^2] < \Var(H) \text{ and } (c_g + c_h)\E[(H-G)^2] < c_h\Var(H).
    \end{equation}
    Under these conditions, we can take square roots on both sides of~\eqref{eq:pstar-squared-interior} to obtain
    \begin{equation}
        p^* = \sqrt{\frac{c_g}{c_h}\frac{1}{\frac{\Var(H)}{\E[(H-G)^2]} - 1}}.
    \end{equation}
    The objective value at this point is
    \begin{equation}
        2\sqrt{c_gc_h} \sqrt{\E[(H-G)^2](\Var(H) - \E[(H-G)^2])}.
    \end{equation}
    Finally, comparing the above objective value with that of $p^*=1$, we have that
    \begin{align}
        & 2\sqrt{c_gc_h} \sqrt{\E[(H-G)^2](\Var(H) - \E[(H-G)^2])} < c_h \left( \Var(H) - \E[(H-G)^2] \right) + c_g\E\left[(H-G)^2\right] \\
        \Longleftrightarrow & 0 < c_h^2 \left( \Var(H) - \E[(H-G)^2] \right)^2 - 2c_gc_h\E[(H-G)^2](\Var(H) - \E[(H-G)^2]) + c_g^2\E\left[(H-G)^2\right]^2 \\
        \Longleftrightarrow & 0 < \left(c_h(\Var(H) - \E[(H-G)^2]) - c_g\E[(H-G)^2]  \right)^2. 
    \end{align}
    Under the condition that $(c_g + c_h)\E[(H-G)^2] < c_h\Var(H)$, the above inequality cannot hold, since the squared term on the right-hand side will always be positive (and nonzero). Thus, we have that
    \begin{equation}
        p^* = \begin{cases}
            \sqrt{\frac{c_g}{c_h}\frac{1}{\frac{\Var(H)}{\E[(H-G)^2]} - 1}} & \text{ if } (c_g + c_h)\E[(H-G)^2] < c_h\Var(H) \text{ and } \E[(H-G)^2] < \Var(H) \\
            1 & \text{ otherwise.}
        \end{cases}
    \end{equation}
    Under the constraint that $c_h \geq c_g$, this simplifies to
    \begin{equation}
        p^* = \begin{cases}
            \sqrt{\frac{c_g}{c_h}\frac{\E[(H-G)^2]}{\Var(H) - \E[(H-G)^2]}} & \text{ if } \E[(H-G)^2] < \frac{\costh}{\costg + \costh} \Var(H) \\
            1 & \text{ otherwise.}
        \end{cases}
    \end{equation}

\end{proof}

\subsection{Proof of Proposition~\ref{prop:optimal-budget}}
\begin{proof}
        Following the simplification of Problem~\eqref{problem:optimal-random} in the proof of  Proposition~\eqref{prop:optimal-random}, Problem~\eqref{problem:optimal-budget} is also equivalent to minimizing the following objective:
    \begin{equation}
        J(\pi) = c_h \E[\pi(X)]\left( \Var(H) - \E[(H-G)^2] + \E\left[(H-G)^2\frac{1}{\pi(X)}\right] \right) + c_g\E\left[(H-G)^2\frac{1}{\pi(X)}\right].
    \end{equation}
    
    At this point, we leverage the discreteness of $\cX$ to write the objective in a simpler form.
    Let $P \in \Delta^{\cX}$ be the probability mass function of $X$, expressed as a vector, and let $I \in \{0,1\}^{|\cX|}$ be the indicator that $X$ takes each value in $\cX$. 
    Furthermore, let $p \in [0,1]^{|\cX|}$ be the vector of $\pi(x)$ for all $x \in \cX$.
    Then, we can express $\pi(X) = p^\top I$ and $\E[\pi(X)] = p^\top P$, and write the objective as
    \begin{align}
    \label{eq:discrete_objective}
        J(\pi) = J(p) = p^\top P\left( \Var(H) - \E[(H-G)^2] + \E\left[(H-G)^2\frac{1}{p^\top I}\right] \right) + \frac{c_g}{c_h}\E\left[(H-G)^2\frac{1}{p^\top I}\right].
    \end{align}
    From here on out, we assume that $P_x > 0$ for all $x \in \cX$. 
    The final result will hold without loss of generality, since the value of the optimal policy on measure-zero points does not change the value of the objective.
    For any $x$, we clearly cannot have $p_x = 0$, otherwise the objective would be infinite.
    This rules out $p_x = 0$ for almost all $x$.
    We are left with the constraint that $p \preceq 1$.

    Forming the Lagrangian,
    \begin{align}
        &\mathcal{L}(p,\lambda) = J(p) + \lambda^\top (p - 1) \\
        &= p^\top P\left( \Var(H) - \E[(H-G)^2] + \E\left[(H-G)^2\frac{1}{p^\top I}\right] \right) + \frac{c_g}{c_h}\E\left[(H-G)^2\frac{1}{p^\top I}\right] + \lambda^\top (p - 1).
    \end{align}
    Taking the gradient with respect to $p$ gives $\nabla_p \mathcal{L}(p, \lambda)$ equal to
    \begin{equation}      
        P\left( \Var(H) - \E[(H-G)^2] \right) -  \left(p^\top P
        + \frac{c_g}{c_h}\right) \E\left[(H-G)^2\frac{I}{(p^\top I)^2}\right]  + P \E\left[(H-G)^2\frac{1}{p^\top I}\right] + \lambda.
    \end{equation}
    Setting the gradient to zero coordinate-wise then gives that for each $x$,
    \begin{equation}
        P_x\left( \Var(H) - \E[(H-G)^2]  \E\left[(H-G)^2\frac{1}{p^\top I}\right]\right) = \left(p^\top P + \frac{c_g}{c_h}\right) \E\left[(H-G)^2\frac{I_x}{p_x^2}\right] - \lambda_x.
    \end{equation}
    By the definition of the conditional expectation, and rearranging, we can rewrite this as 
    \begin{equation}
        \Var(H) - \E[(H-G)^2]  +  \E\left[(H-G)^2\frac{1}{p^\top I}\right] + \frac{\lambda_x}{P_x} = \left(p^\top P + \frac{c_g}{c_h}\right) \E\left[(H-G)^2\frac{1}{p_x^2} \mid X = x\right] .
    \end{equation}
    Solving for the optimal value as a function of the Lagrange multipliers $\lambda$ gives the following expression:
    \begin{equation}
        p_x(\lambda)^2 = \frac{\left(p^\top P + \frac{c_g}{c_h}\right) \E\left[(H-G)^2 \mid X = x\right]}{\Var(H) - \E[(H-G)^2] + \E\left[(H-G)^2\frac{1}{p^\top I}\right] + \frac{\lambda_x}{P_x}}.
    \end{equation}
    The denominator of this expression is always positive, since for all valid $p$, $\frac{(H-G)^2}{p^\top I} \asgeq (H-G)^2$, and the remaining terms are positive.
    Thus,
    \begin{equation}
        \label{eq:p_x_lambda}
        p_x(\lambda) = \sqrt{\frac{\left(p^\top P + \frac{c_g}{c_h}\right) \E\left[(H-G)^2 \mid X = x\right]}{\Var(H) - \E[(H-G)^2] + \E\left[(H-G)^2\frac{1}{p^\top I}\right] + \frac{\lambda_x}{P_x}}}.
    \end{equation}
    Next, we require some detailed case-by-case analysis.

    \textbf{Case 1: The Interior.} First, we handle the case when the constraint is inactive, i.e., for any fixed $\lambda$, $p_x(\lambda) \in (0,1)$. (If no such $x$ exists, then the solution is trivially $p(\lambda) = \mathbf{1}_{|\cX|}$.)  For any $x$ such that $p_x(\lambda)$ is in the interior, by complementary slackness, $\lambda_x = 0$. Now, for any $x' \in \cX$ satisfying $p_{x'}(\lambda) \in (0,1)$, we can write 
    \begin{equation}
        \frac{p_x(\lambda)}{p_{x'}(\lambda)}  = \sqrt{\frac{\E\left[(H-G)^2 \mid X = x\right]}{\E\left[(H-G)^2 \mid X = x'\right]}},
    \end{equation}
    simply by applying~\eqref{eq:p_x_lambda} to $x$ and $x'$, then dividing these expressions.
    This tells us that for all $\lambda$ and all $x$ such that $p_x(\lambda)$ is in the interior, $p_x(\lambda) = \gamma u(x)$ for some as-yet-unknown $\gamma \in \left(0, \frac{1}{\sup_{x: p_x \in (0, 1)} u(x)}\right]$.
    Because $\lambda_x=0$ on these $x$,  the solution to the optimization problem must have the same property. 

    \textbf{Case 2: The Boundary.} When the constraint is active, $p_x(\lambda) = 1$, since $p_x(\lambda) =0$ is almost always impossible, as established earlier.
    Examining~\eqref{eq:p_x_lambda} shows us that the constraint only activates in the case that $u(x) = \E[(H-G)^2 \mid X = x]$ is too large:
    \begin{equation}
        p_x = 1 \Longleftrightarrow u(x) \geq \sqrt{\frac{\Var(H) - \E[(H-G)^2] + \E\left[ (H-G)^2 \frac{1}{p^\top I}\right]}{p^\top P + \frac{c_g}{c_h}}} = \tau(p),
    \end{equation}
    since in the alternate case, the unconstrained solution lies in the interior.
    The Lagrange multiplier $\lambda_x$, in this case, takes on the value such that $p_x(\lambda_x) = 1$; a non-negative such value always exists by virtue of the fact that $u(x)$ is sufficiently large.

    Combining Case 1 and Case 2 tells us that our optimal policy has the form
    \begin{equation}
        \pi(x) = \begin{cases}
            \gamma\sqrt{u(x)} & \sqrt{u(x)} \leq \tau \\
            1 & \text{otherwise}
        \end{cases},
    \end{equation}
    for a $\tau \in \mathbb{R}_{>0}$ and $\gamma \in \left(0, \inf_{x \colon u(x)\leq \tau^2} \sqrt{u(x)}\right] $, which we assume w.l.o.g. is equivalent to $\gamma \in \left(0, \frac{1}{\tau}\right)$. The constraint on $\gamma$ is necessary, as otherwise we can have  $\pi(x) > 1$, which is a contradiction.\looseness=-1
    
    Note that another way to express this policy is as $p_x = \ind{u(x) > \tau^2} + \gamma\sqrt{u(x)}\ind{u(x)\leq \tau^2}$.
    With this in mind, and defining the vector $U$ with entries $U_x = \E[(H-G)^2 \mid X = x]$ and $W = \Var(H) - \E[(H-G)^2]$ we can rewrite the objective in \eqref{eq:discrete_objective} as
    \begin{align}
        J(p) &= \left(\sum\limits_{x \in \cX} p_x P_x \right)\left(W + \E\left[(H-G)^2\frac{1}{p_X}\right] \right) + \frac{c_g}{c_h}\E\left[(H-G)^2\frac{1}{p_X}\right] 
    \end{align}
    which is equivalent to
    \begin{align}
        J(\gamma, \tau) &=  \E\left[\ind{U_X > \tau^2} + \gamma\sqrt{U_X}\ind{U_X \leq \tau^2}\right]  \\  &\hspace{1cm}\times \left(W + \E\left[(H-G)^2 \ind{U_X > \tau^2} \right] + \E\left[\frac{(H-G)^2}{\gamma \sqrt{U_X}} \ind{U_X \leq \tau^2} \right]\right) \\
        &\hspace{1cm} + \frac{c_g}{c_h}\left(\E\left[(H-G)^2 \ind{U_X > \tau^2} \right] + \E\left[\frac{(H-G)^2}{\gamma \sqrt{U_X}} \ind{U_X \leq \tau^2}\right]\right) 
    \end{align}

    This objective is convex in $\gamma$, but not differentiable or convex in $\tau$.
    For that reason, we will solve for the optimal $\gamma$ as a function of $\tau$ subject to the constraint that $\gamma > 0$ and $\gamma \sqrt{u(x)} \leq 1$ $\forall x$ where $u(x) \leq \tau^2$, and our algorithm will search over $\tau$ to complete the optimization.
    Keeping only terms with a dependence on $\gamma$, and recognizing that $\E\left[\frac{(H-G)^2}{\sqrt{U_X}}\right] = \E\left[\sqrt{U_X}\right]$ gives the expression\looseness=-1
    \begin{equation}
    \label{eq:objective}
    \begin{split}
    &\E\left[\sqrt{U_X}\ind{U_X \leq \tau^2}\right] \times \\
    &\hspace{1cm}\left[\frac{1}{\gamma}\left(\E\left[\ind{U_X > \tau^2}\right] + \frac{\costg}{\costh}\right) + \gamma \left(W + \E\left[(H - G)^2\ind{U_X > \tau^2}\right]\right)\right]
    \end{split}
    \end{equation}

    Once again, we know that the optimal solution as a function of $\tau$, $\gamma^*(\tau)$ lies either on the boundary or the interior, and we will compare the values of the objective in both cases. In the case that $\gamma^*(\tau) \in (0, \tau^{-1})$, $\gamma^*(\tau)$ is a critical point, thus
    differentiating and setting equal to zero gives that
    \begin{equation}
      \frac{1}{\gamma^*(\tau)^2} \left(\frac{\costg}{\costh} + \P\left(U_X > \tau^2\right)\right) = W + \E\left[(H - G)^2\ind{U_X > \tau^2}\right],
    \end{equation}
    and thus,
    \begin{equation}
    \label{eq:gamma-star}
        \gamma^*(\tau)^2 = \frac{\frac{\costg}{\costh} + \P\left(U_X > \tau^2\right)}{W + \E\left[(H - G)^2\ind{U_X > \tau^2}\right]} = \frac{\frac{\costg}{\costh} + \P\left(U_X > \tau^2\right)}{\Var(H) - \E\left[(H - G)^2\ind{U_X \leq \tau^2}\right]}.
    \end{equation}
    As in the proof of Proposition~\ref{prop:optimal-random}, because we are in the case $\gamma^* \in (0, \tau^{-1})$, the right-hand side must be positive and it cannot be greater than $\tau^{-1}$. This gives us that
    \begin{equation}
        \gamma^*(\tau) \in (0, \tau^{-1}) \Longrightarrow  \E\left[(H - G)^2\ind{U_X \leq \tau^2}\right] < \Var(H) 
    \end{equation}
    and
       \begin{equation}
        \frac{\frac{\costg}{\costh} + \P\left(U_X > \tau^2\right)}{\Var(H) - \E\left[(H - G)^2\ind{U_X \leq \tau^2}\right]} < \frac{1}{\tau^2},
    \end{equation}
    and under these conditions we can take square roots on both sides of \eqref{eq:gamma-star} to obtain 
        \begin{equation}
        \label{eq:gamma-star-sqrt}
        \gamma^*(\tau) = \sqrt{\frac{\frac{\costg}{\costh} + \P\left(U_X > \tau^2\right)}{\Var(H) - \E\left[(H - G)^2\ind{U_X \leq \tau^2}\right]}} = \sqrt{\frac{\frac{\costg}{\costh} + \P\left(U_X > \tau^2\right)}{\Var(H) - \E\left[U_X\ind{U_X \leq \tau^2}\right]}}.
    \end{equation}

Comparing the objective value with $\gamma^*(\tau) = \tau^{-1}$ vs \eqref{eq:gamma-star-sqrt}, we know that \eqref{eq:objective} is decreasing in $\gamma$ for $0 < \gamma < \sqrt{\frac{\frac{\costg}{\costh} ~+~ \P\left(U_X > \tau^2\right)}{\Var(H) - \E\left[U_X\ind{U_X \leq \tau^2}\right]}}$. Thus, we have that
\begin{equation}
        \gamma^*(\tau) = \min\left(
        \sqrt{\frac{{c_g}/{c_h} + \mathbb{P}\left(U_X > \tau^2\right)}{\big(\mathrm{Var}(H) - \mathbb{E}[U_X\ind{U_X \leq \tau^2}]\big)_+}}, % & \text{\small if  $\E\left[U\ind{ U \leq \tau^2}\right] < \Var(H) - \tau^2 \left(\frac{\costg}{\costh} + \P\left(U > \tau^2\right)\right)$} \\
        \frac{1}{\tau} \right).
    \end{equation}
Plugging into the original objective $J(\tau, \gamma^*(\tau))$ and minimizing over $\tau$ yields the solution.
\end{proof}

\subsection{Proof of Proposition~\ref{prop:optimal-random-integer}}
\begin{proof}
    Since $\Error_T(\pi)$ is monotone in $T$ for all $\pi$, we should first set $T^{\rm stop}$ to be the largest $T$ for which the constraint holds.
    This value is
    \begin{equation}
        T^{\rm stop} = \left\lfloor \frac{B}{c_h p + c_g} \right\rfloor.
    \end{equation}
    Plugging this into the objective yields
    \begin{equation}
        \frac{1}{\left\lfloor \frac{B}{c_h p + c_g} \right\rfloor}\left( \Var(H) - \E[(H-G)^2] + \frac{1}{p}\E\left[(H-G)^2\right]\right).
    \end{equation}
    This is a complicated optimization problem because of the floor function, and cannot be solved by setting the gradient to zero.
    We will begin by searching over all values of $p \in (0, 1]$ for which $\frac{B}{c_h p + c_g} = k$ for $k \in \mathbb{N}_+$, i.e., $p \in \left\{\frac{B-kc_g}{kc_h} : k \in \{\lceil B / (\costh + \costg)\rceil, \ldots, \lfloor B/\costg\rfloor\}\right\}$.
    In terms of $k$, and denoting $E = \E[(H-G)^2]$ and $V=\Var(H)-\E[(H-G)^2]$, the objective then becomes
    \begin{equation}
        \frac{1}{k}\left( V + \frac{kc_h}{B-kc_g}E\right) = \frac{1}{k}V + \frac{c_h}{B - kc_g}E.
    \end{equation}
    Ignoring the discreteness of $k$, in the case that $p^* \in (0, 1)$ we can set the derivative to zero, getting
    \begin{align}
        & \frac{c_g c_h}{(B-kc_g)^2}E = \frac{1}{k^2}V \\
        \Longleftrightarrow & k^2 c_gc_h\frac{E}{V} = (B-kc_g)^2 \\
        \Longleftrightarrow & k^2 \left(c_g^2 - c_gc_h\frac{E}{V}\right)  - 2kc_gB + B^2 = 0
    \end{align}
    The positive solution to this quadratic is
    \begin{equation}
        k = \frac{2c_gB + \sqrt{4c_g^2B^2-4B^2\left(c_g^2 - c_gc_h\frac{E}{V}\right)}}{2\left(c_g^2 - c_gc_h\frac{E}{V}\right)} = B\frac{1 + \sqrt{ \frac{c_h}{c_g}\frac{E}{V}}}{c_g - c_h\frac{E}{V}}.
    \end{equation}
    Thus, the optimal $k^*$ solves the following optimization problem: 
    \begin{equation}
        k^* = \argmin_{k \in \left\{  \left\lfloor B\frac{1 + \sqrt{ \frac{c_h}{c_g}\frac{E}{V}}}{c_g - c_h\frac{E}{V}}\right\rfloor ,  \left\lceil B\frac{1 + \sqrt{ \frac{c_h}{c_g}\frac{E}{V}}}{c_g - c_h\frac{E}{V}}\right\rceil\right\}}\frac{1}{k}V + \frac{c_h}{B - kc_g}E,
    \end{equation}
    And the optimal $p^*$ is either
    \begin{equation}
        p^* = \frac{B-k^*c_g}{k^*c_h}
    \end{equation}
    or the boundary solution $p^*=1$.
    To disambiguate between the two, we can directly compute the objective value for each.
\end{proof}

\subsection{Proof of Proposition~\ref{prop:m-estimator-mse-dist}}
\begin{proof}
    The asymptotic normality statement can be read off as a simplified version of Theorem 1 from~\cite{zrnic2024active}.
    The second part follows because if $Z \sim \cN(0, \Sigma^*)$, then $\|Z\|^2_2 = \|(V^*)^{-1/2}Z\|^2_2$, where $V^*$ is the eigenvector matrix of $\Sigma^*$ (since $(V^*)^{-1/2}$ is unitary).
    Thus, taking $\Lambda^*$ to be the (diagonal) eigenvalue matrix of $\Sigma^*$ and defining we have that $\|Z\|^2_2 \eqd \|\Lambda Z'\|^2_2$, where $Z' \sim \N(0,\mathbf{I}_d)$.
    Since $\|\Lambda Z'\|^2_2 = \sum\limits_{j=1}^d \lambda_j (Z'_j)^2$, and $Z'_j \iidsim \chi^2_1$, the result is proven.
\end{proof}

\subsection{Proof of Proposition~\ref{prop:convex-budget}}
\begin{proof}
    Following the simplification of Problem~\eqref{problem:optimal-budget}, our problem is equivalent to minimizing the following objective:
    \begin{equation}
    \label{eq:M-objective}
        (c_h \E[\pi(X)] + c_g) \Tr(\Sigma^*).
    \end{equation}
    Expanding out $\Sigma^*$, we can write
    \begin{equation}
        \Tr(\Sigma^*) = \Tr\left(W_{\theta^*}^{-1}\Var\left( \nabla \ell_{\theta^*}^g + \left( \nabla \ell_{\theta^*} - \nabla \ell_{\theta^*}^g \right) \frac{\xi}{\pi(X)} \right) W_{\theta^*}^{-1}\right)
    \end{equation}
    Expanding out the variance gives 
    \begin{align}
        &\Var\left( \nabla \ell_{\theta^*}^g + \left( \nabla \ell_{\theta^*} - \nabla \ell_{\theta^*}^g \right) \frac{\xi}{\pi(X)} \right) \\
        = &\E\left[ \left(\nabla \ell_{\theta^*}^g + \left( \nabla \ell_{\theta^*} - \nabla \ell_{\theta^*}^g \right) \frac{\xi}{\pi(X)}\right)\left(\nabla \ell_{\theta^*}^g + \left( \nabla \ell_{\theta^*} - \nabla \ell_{\theta^*}^g \right) \frac{\xi}{\pi(X)}\right)^\top \right] - \E[\nabla \ell_{\theta^*}]\E[\nabla \ell_{\theta^*}]^\top.
    \end{align}
    Expanding out the squared term yields 
    \begin{align}
        &\E\left[ \left(\nabla \ell_{\theta^*}^g + \left( \nabla \ell_{\theta^*} - \nabla \ell_{\theta^*}^g \right) \frac{\xi}{\pi(X)}\right)\left(\nabla \ell_{\theta^*}^g + \left( \nabla \ell_{\theta^*} - \nabla \ell_{\theta^*}^g \right) \frac{\xi}{\pi(X)}\right)^\top \right] \\
        &\qquad = \E\left[\nabla \ell_{\theta^*}^g (\nabla \ell_{\theta^*}^g)^\top \right] \\
        & \qquad \qquad + \E\left[\frac{\xi}{\pi(X)} \left(\nabla \ell_{\theta^*}^g\left( \nabla \ell_{\theta^*} - \nabla \ell_{\theta^*}^g \right)^\top + \left( \nabla \ell_{\theta^*} - \nabla \ell_{\theta^*}^g \right)(\nabla\ell_{\theta^*}^g)^\top \right) \right] \\
        & \qquad \qquad + \E\left[ \left(\left( \nabla \ell_{\theta^*} - \nabla \ell_{\theta^*}^g \right) \frac{\xi}{\pi(X)}\right)\left(\left( \nabla \ell_{\theta^*} - \nabla \ell_{\theta^*}^g \right) \frac{\xi}{\pi(X)}\right)^\top\right] \\
        &\qquad  = \E\left[\nabla \ell_{\theta^*}^g (\nabla \ell_{\theta^*}^g)^\top \right] \\
        & \qquad \qquad + \E\left[\nabla \ell_{\theta^*}^g\left( \nabla \ell_{\theta^*} - \nabla \ell_{\theta^*}^g \right)^\top + \left( \nabla \ell_{\theta^*} - \nabla \ell_{\theta^*}^g \right)(\nabla\ell_{\theta^*}^g)^\top \right] \\
        & \qquad \qquad + \E\left[ \frac{1}{\pi(X)} \left(\left( \nabla \ell_{\theta^*} - \nabla \ell_{\theta^*}^g \right) \right)\left(\left( \nabla \ell_{\theta^*} - \nabla \ell_{\theta^*}^g \right)\right)^\top\right].
    \end{align}
    Thus, by the linearity of the $\Tr$ operator, we can rewrite the trace as $\Tr(\Sigma^*) = \E\left[\frac{M}{\pi(X)}\right] + C$, where
    \begin{equation}
        M =  \Tr\left(W_{\theta^*}^{-1}\left( \nabla \ell_{\theta^*} - \nabla \ell_{\theta^*}^g \right)\left(\nabla \ell_{\theta^*} - \nabla \ell_{\theta^*}^g \right)^\top W_{\theta^*}^{-1}\right)
    \end{equation}
    and $C$ is
    \begin{align}
        &\Tr\Big(W_{\theta^*}^{-1}\Big(\E\left[\nabla \ell_{\theta^*}^g (\nabla \ell_{\theta^*}^g)^\top + \nabla \ell_{\theta^*}^g\left( \nabla \ell_{\theta^*} - \nabla \ell_{\theta^*}^g \right)^\top + \left( \nabla \ell_{\theta^*} - \nabla \ell_{\theta^*}^g \right)(\nabla\ell_{\theta^*}^g)^\top\right] - \E[\nabla \ell_{\theta^*}]\E[\nabla \ell_{\theta^*}]^\top\Big) W_{\theta^*}^{-1}\Big) \\
        &= \Tr\Big(W_{\theta^*}^{-1}\Big(\E\left[\nabla \ell_{\theta^*}^g (\nabla \ell_{\theta^*})^\top + (\nabla \ell_{\theta^*} - \nabla \ell_{\theta^*}^g)(\nabla \ell_{\theta^*}^g)^\top \right] - \E[\nabla \ell_{\theta^*}]\E[\nabla \ell_{\theta^*}]^\top\Big) W_{\theta^*}^{-1}\Big).
    \end{align}
    Returning to the objective, and excluding factors that do not depend on $\pi$, we can write it now as
    \begin{align}
        &(c_h \E[\pi(X)] + c_g) \left(\E\left[\frac{M}{\pi(X)}\right] + C\right) 
        \propto_{\pi}(c_h \E[\pi(X)] + c_g) \E\left[\frac{M}{\pi(X)}\right] +c_h \E[\pi(X)] C.
    \end{align}
    In discrete form, following Propostion~\ref{prop:optimal-budget}, this is equivalent to
    \begin{equation}
        (c_h p^\top P + c_g) \E\left[\frac{M}{p^\top I}\right] + c_h p^\top P C.
    \end{equation}
    Taking the derivative with respect to $p$ and setting it to zero coordinatewise yields
    \begin{equation}
        c_h P_x\E\left[\frac{M}{p^\top I}\right] + c_h P_x C = (c_h p^\top P + c_g) \E\left[M I_x\right],
    \end{equation}
    and thus,
    \begin{equation}
        p_x = \sqrt{\frac{(c_h p^\top P + c_g) \E\left[M \mid X = x\right]}{c_h \E\left[\frac{M}{p^\top I}\right] + c_h C}} \propto_x \sqrt{\E\left[M \mid X = x\right]} = \sqrt{U(x)}.
    \end{equation}
    Plugging $\pi(x) = \gamma \sqrt{\E\left[M \mid X = x\right]}$ back into \eqref{eq:M-objective} gives the one-dimensional objective
    \begin{equation}
        \frac{c_g}{\gamma}\E\left[\frac{M}{\sqrt{\E\left[M \mid X = x\right]}}\right]  + c_h \gamma \E[\sqrt{\E\left[M \mid X = x\right]}] C.
    \end{equation}
    The tower property gives us that $\E\left[\frac{M}{\sqrt{\E\left[M \mid X = x\right]}}\right] = \E\left[\sqrt{\E\left[M \mid X = x\right]}\right]$, yielding the objective
    \begin{equation}
        \frac{c_g}{\gamma}\E\left[\sqrt{\E\left[M \mid X = x\right]}\right]  + c_h \gamma \E[\sqrt{\E\left[M \mid X = x\right]}] C,
    \end{equation}
    which is equivalent to minimizing
    \begin{equation}
        \frac{c_g}{\gamma}  + c_h \gamma  C.
    \end{equation}
    The solution to this problem is
    \begin{equation}
        \gamma^* = \sqrt{\frac{c_g}{c_h} \cdot \frac{1}{C}}.
    \end{equation}
\end{proof}

\subsection{Proof of Proposition~\ref{prop:errors-u}}
\begin{proof}
Following the derivation in Section~\ref{sec:derivation}, we have that for any $\pi$
\begin{equation}
    \Var(\Delta^\pi) = \Var(H) - \E[(H-G)^2] + \E\left[(H-G)^2\frac{1}{\pi(X)}\right].
\end{equation}

% \begin{equation}
% \begin{aligned}
% \Error_{T^{\rm stop}}(\tilde\pi)-\Error_{T^{\rm stop}}(\pi^*)
% &= c_h\!\bigl(\E[\tilde\pi(X)]-\E[\pi^*(X)]\bigr)\!\left(\Var(H)-\E[(H-G)^2]\right) \\
% &\quad +\bigl(c_h\E[\tilde\pi(X)]+c_g\bigr)\,\E\!\left[\frac{(H-G)^2}{\tilde\pi(X)}\right] \\
% &\qquad -\bigl(c_h\E[\pi^*(X)]+c_g\bigr)\,\E\!\left[\frac{(H-G)^2}{\pi^*(X)}\right].
% \end{aligned}
% \end{equation}
We then immediately get that
\begin{align}
    \Var(\Delta^{\tilde{\pi}})-\Var(\Delta^{\pi^*}) = \E\left[\frac{(H-G)^2}{\tilde\pi(X)}-\frac{(H-G)^2}{\pi^*(X)}\right] &\leq b\E\left[\frac{1}{\tilde\pi(X)}-\frac{1}{\pi^*(X)}\right] 
    \leq b\delta.
\end{align}
\end{proof}

\subsection{Proof of Corollary~\ref{cor:additive_error}}
\begin{proof}
Since 
\begin{equation}
    \tilde\pi(x) = (\gamma^*+\delta_\gamma)\sqrt{U(x)+\delta_U(x)},
\end{equation}
we have
\begin{equation}
    \frac{1}{\tilde\pi(x)} = \frac{1}{(\gamma^*+\delta_\gamma)\sqrt{U(x)+\delta_U(x)}}
    = \frac{1}{\gamma^*\sqrt{U(x)}}
    \frac{1}{\left(1+\frac{\delta_\gamma}{\gamma^*}\right)\sqrt{1+\frac{\delta_U(x)}{U(x)}}}.
\end{equation}
A first-order Taylor expansion yields
\begin{equation}
    \frac{1}{\left(1+\frac{\delta_\gamma}{\gamma^*}\right)\sqrt{1+\frac{\delta_U(x)}{U(x)}}}
    = 1 - \frac{\delta_\gamma}{\gamma^*} - \frac{1}{2}\frac{\delta_U(x)}{U(x)} + o\Bigl(\delta_\gamma,\tfrac{\delta_U(x)}{U(x)}\Bigr).
\end{equation}
Thus,
\begin{equation}
    \frac{1}{\tilde\pi(x)} - \frac{1}{\pi^*(x)}
    = \frac{-\delta_\gamma}{(\gamma^*)^2\sqrt{U(x)}}
    -\frac{1}{2\gamma^*}\frac{\delta_U(x)}{U(x)^{3/2}}
    + o\Bigl(\delta_\gamma,\tfrac{\delta_U(x)}{U(x)}\Bigr).
\end{equation}
Ignoring second-order terms, since $U(x) \ge \epsilon$ almost surely, we have
\begin{equation}
    \Bigl|\frac{1}{\tilde\pi(x)} - \frac{1}{\pi^*(x)}\Bigr|
    \le \frac{|\delta_\gamma|}{(\gamma^*)^2\sqrt{\epsilon}}
    +\frac{1}{2\gamma^*\epsilon^{3/2}}\,|\delta_U(x)|.
\end{equation}
Taking the expectation and using linearity,
\begin{equation}
    \E\!\left[\frac{1}{\tilde\pi(X)} - \frac{1}{\pi^*(X)}\right]
    \le \frac{|\delta_\gamma|}{(\gamma^*)^2\sqrt{\epsilon}}
    +\frac{1}{2\gamma^*\epsilon^{3/2}}\,\E[\,|\delta_U(X)|].
\end{equation}
Plugging this bound into the initial inequality for $\Var(\Delta^{\tilde{\pi}})$ completes the proof:
\begin{equation}
    \Var(\Delta^{\tilde{\pi}})-\Var(\Delta^{\pi^*})  \le b\left(\frac{|\delta_\gamma|}{(\gamma^*)^2\sqrt{\epsilon}}
    +\frac{1}{2\gamma^*\epsilon^{3/2}}\,\E[\,|\delta_U(X)|]\right).
\end{equation}
\end{proof}

\subsection{Proof of Proposition~\ref{prop:optimal-sampling}}
\begin{proof}
    We will borrow notation from the proof of Proposition~\ref{prop:optimal-budget}, and express all quantities in vector form.
    The optimization problem in~\eqref{problem:optimal-sampling} only depends on $Q$ through the likelihood ratio, $\frac{dP}{dQ} = r \in \R_{>0}^{|\cX|}$, and $Q,P$ are absolutely continuous with respect to one another.
    So, we will learn $r$ and then calculate $Q^* = P/r$.
    
    Ignoring terms that do not depend on $r$, the problem in~\eqref{problem:optimal-sampling} can be rewritten as
    \begin{equation}
        \begin{aligned}
            \minimize_{\substack{r \in \R_{>0}^{|\cX|}}} \quad &r^\top \E_{P}\left[I \left(H^2 + \left(\frac{1}{\pi(X)} - 1\right)(H-G)^2 \right)\right] \\
            \st \quad & (1/r)^\top P = 1.
        \end{aligned}
    \end{equation}
    Forming the Lagrangian,
    \begin{equation}
        \mathcal{L}(r,\lambda) = r^\top \E_{P}\left[I \left(H^2 + \left(\frac{1}{\pi(X)} - 1\right)(H-G)^2 \right)\right] + \lambda((1/r)^\top P-1).
    \end{equation}
    Taking the gradient gives
    \begin{equation}
        \nabla_r \mathcal{L}(r,\lambda) = \E_{P}\left[I \left(H^2 + \left(\frac{1}{\pi(X)} - 1\right)(H-G)^2 \right)\right] - \lambda P/(r^2),
    \end{equation}
    and setting it to zero yields 
    \begin{equation}
        r^*_x \propto_x \sqrt{\frac{1}{\E_{P}\left[\left(H^2 + \left(\frac{1}{\pi(X)} - 1\right)(H-G)^2 \right) \Big\vert X = x\right]}} = \sqrt{\frac{1}{\nu_x}}.
    \end{equation}
    To ensure the proper normalization, we set
    \begin{equation}
        r^*_x = \frac{\sqrt{\nu}^\top P}{\sqrt{\nu_x}}.
    \end{equation}
    Thus, $Q^*(x) = P/r^* = \frac{\sqrt{\nu_x} P_x}{\sqrt{\nu}^\top P}$.
\end{proof}
\section{Additional empirical results}
\label{app:additional_results}

\subsection{Bernoulli data}

\begin{figure}[!h]
    \centering
    \includegraphics[width=1\linewidth]{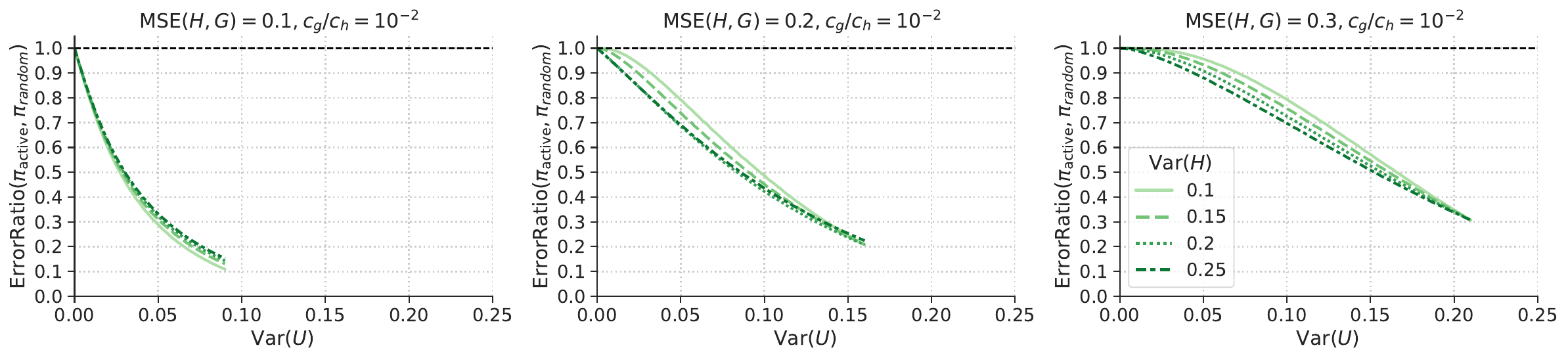}
    \vspace{-14pt}
    \caption{Results on the Bernoulli data (\S\ref{sec:bernoulli_data}) for $\pi_\mathrm{active}$ vs. $\pi_\mathrm{random}$ while varying  $\MSE(H, G)$  and $\Var(U)$. As in Figure~\ref{fig:gaussian}, each line plots a different value of $\Var(H)$, where we choose values  that are representative of low, medium, or high variance settings  compared to $\MSE(H, G)$.}
    \label{fig:bernoulli}
\end{figure}

Figure~\ref{fig:bernoulli} provides results for the Bernoulli data setting in Section~\ref{sec:bernoulli_data} when comparing $\pi_\mathrm{active}$ to $\pi_\mathrm{random}$. Recall that here the results differ from comparing to $\pi_\mathrm{base}$ only when $\MSE(H, G) < \frac{\costh}{\costh + \costg} \Var(H)$, as otherwise the optimal sampling rate for $\pi_\mathrm{random}$ is simply $p^* = 1$.

\subsubsection{On the error ratio lower bound} 
It is interesting to observe that  $\ErrorRatio(\pi_\mathrm{active}, \pi_\mathrm{base})$ is lower-bounded in the Bernoulli data setting at a value close to $\MSE(H, G)$. To see why, we note that the lowest value of $\ErrorRatio(\pi_\mathrm{active}, \pi_\mathrm{base})$ is obtained when $U$ is maximum variance---which is achieved when $U$ is a binary random variable that is $1$ when $G \neq H$, and $0$ otherwise. Recall that in the Bernoulli data setting both $H$ and $G$ are binary, and $\MSE(H, G) = \mathbb{P}(H \neq G)$. We can then compute $\ErrorRatio(\pi_\mathrm{active}, \pi_\mathrm{base})$ after optimizing over $\tau$ as approaching 
\begin{align}
 &\min\left( 
\begin{array}{c}
    \left(\gamma \MSE(H, G) + \frac{\costg}{\costh}\right)\left(1 +(\frac{1}{\gamma} - 1)\frac{\MSE(H, G)}{\Var(H)}\right)  \\[.5em]
     \MSE(H, G) + \frac{\costg}{\costh} 
\end{array}\right) \\[.5em] &\text{ where } \gamma = \sqrt{\frac{\costg/\costh}{(\Var(H) - \MSE(H, G))_+}}. 
\end{align}
Note that we have the first quantity only when $\MSE(H, G) \leq \Var(H) + \costg/\costh$.

\begin{proof}[Derivation.] When $U \rightarrow \ind{H \neq G} \in \{0, 1\}$, from Proposition~\eqref{prop:optimal-budget} $\pi_\mathrm{active}$ approaches either
\begin{equation}
\pi_\mathrm{clip}(x, \tau = 1) = \begin{cases}
\gamma^*(1) & \text{if $h(x) \neq g(x)$} \\
0 & \text{otherwise}
\end{cases} \quad \text{or} \quad
\pi_\mathrm{clip}(x, \tau = 0) = \begin{cases}
1 & \text{if $h(x) \neq g(x)$} \\
0 & \text{otherwise}
\end{cases}
\end{equation}
where $\gamma^*(1) = \sqrt{\frac{\costg/\costh}{(\Var(H) - \MSE(H, G))_+}}$. Plugging these values into the optimization over $\tau \in \{0, 1\}$,
\begin{equation}
        \tau^* = \argmin_{\tau \in \{0, 1\}}\left(\costh\mathbb{E}[\pi_\mathrm{clip}(x; \tau)] + \costg \right) \left(\Var(H) + \mathbb{E}\left[U \left(\pi_\mathrm{clip}(x; \tau)^{-1}  - 1\right)\right]\right),
\end{equation}
at $\tau = 1$ we get
\begin{equation}
\begin{gathered}
    \left(\costh \gamma^*(1) \MSE(H, G) + \costg\right)\left(\Var(H) + \left(\frac{1}{\gamma^*(1)} - 1\right)\MSE(H, G)\right),
\end{gathered}
\end{equation}
and at $\tau = 0$ we get
\begin{equation}
    (\costh \MSE(H, G) + \costg)\Var(H),
\end{equation}
so the optimal $\tau^*$ is the smaller of the two. 
Dividing each by $\costh \Var(H)$ and taking the minimum gives the result for $\ErrorRatio(\pi_\mathrm{active}, \pi_\mathrm{base})$.
\end{proof}

A similar calculation can also be made for $\ErrorRatio(\pi_\mathrm{active}, \pi_\mathrm{random})$, with different bounds for when $\MSE(H, G) \leq \Var(H) - \frac{\costg}{\costh}$ and/or  $\MSE(H, G) \leq \frac{\costh}{\costg + \costh}\Var(H)$ (i.e., both conditions, one or the other condition, or neither condition).\looseness=-1
\subsection{Real data}

We provide experimental results on two additional datasets:

\textbf{ImageNet.} The {ImageNet} dataset~\cite{deng2009imagenet} categorizes input images into one of $1000$  classes. Our goal is to evaluate the accuracy $\E[H]$ of a pretrained ResNet model~\cite{he2016deep}, where $H$ is the binary indicator of whether the model's prediction matches the human label for a given image $X$.  $G$ is the softmax value the model assigns to its predicted class. $U$ is computed as $G(1-G)$.

\textbf{Seahorse.} The Seahorse dataset~\cite{clark-etal-2023-seahorse} focuses on multilingual summarization. We focus on the ``attribution to the source document'' metric for summaries produced by a finetuned 13B parameter mT5 model~\cite{xue2020mt5}.  $H$ comes from human ratings. $G$ is the probability score from a finetuned mT5-XXL autorater model assessing attribution.\footnote{This checkpoint is available at \\ \url{https://huggingface.co/collections/google/seahorse-release-6543b0c06d87d83c6d24193b}}  $U$ is computed as $G(1-G)$.

 \begin{figure}[!h]
    \centering
    \includegraphics[width=\linewidth]{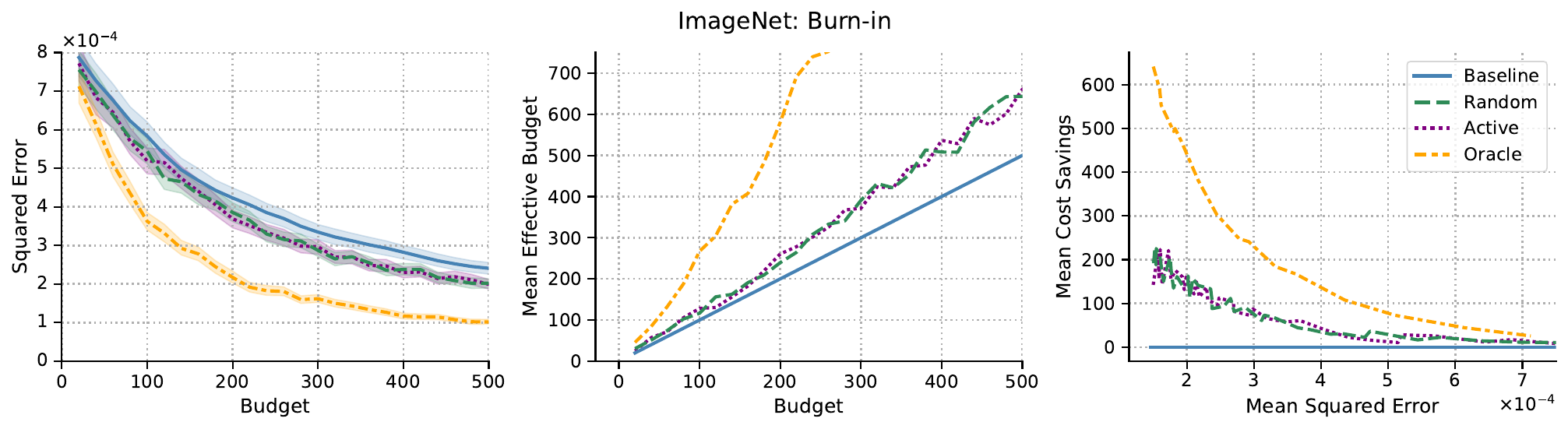} 
    \includegraphics[width=\linewidth]{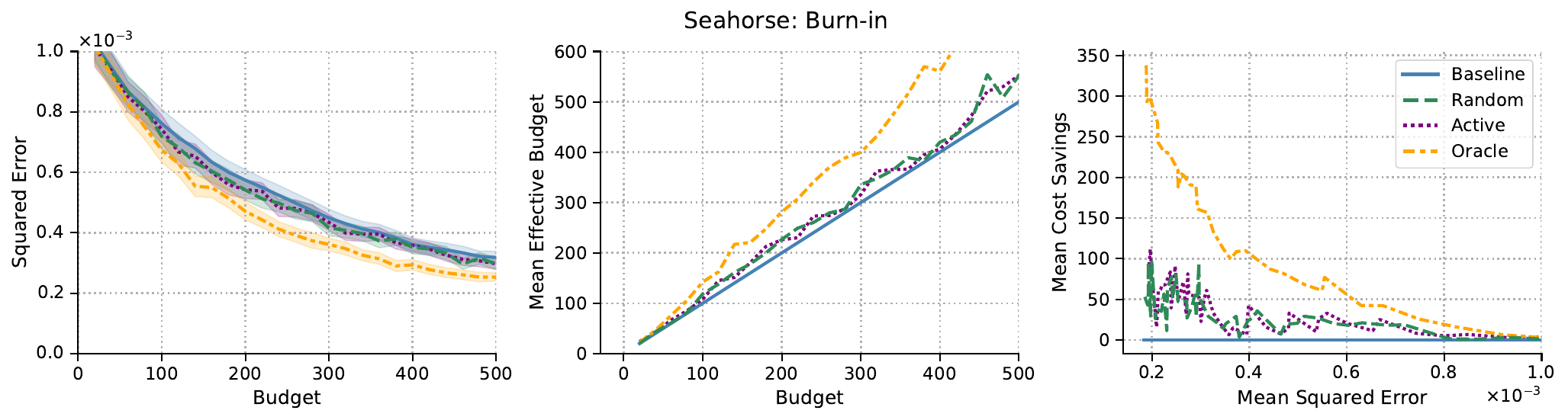}
    \vspace{-14pt}
    \caption{Results on the ImageNet and Seahorse datasets using $200$ examples as a burn-in (approach A2 in Section~\ref{sec:estimating-policies}). The budget on the x-axis  reflects ``additional'' budget used \emph{after} the burn-in examples.\looseness=-1}
    \label{fig:extra_results}
\end{figure}

Results are shown in Figure~\ref{fig:extra_results}, with similar takeaways as the other burn-in (approach A2) experiments in Section~\ref{sec:experimental-results}. For ImageNet, both $\pi_\mathrm{active}$ and $\pi_\mathrm{random}$ substantially outperform $\pi_\mathrm{base}$; however, the estimated $\pi_\mathrm{active}$ still leaves a significant amount of headroom behind with respect to the oracle active policy, and has comparable performance to $\pi_\mathrm{random}$. The Seahorse dataset is an interesting case where the weak rater $G$ is simply not that good, even conditionally. This results in small (but still positive) gains for both the active and random policies---even when $\pi_\mathrm{active}$ uses oracle parameters.\looseness=-1

\section{Implementation details}
\label{app:autorater-details}

All real data experiments in this paper were performed retrospectively with $G$ and $H$ computed once offline for all inputs $x$ in each dataset. Pretrained models and labels for $G$ and $H$, respectively, were used for all datasets except Chatbot Arena (Section~\ref{sec:estimating-policies}). All subsequent experiments for active sampling were then performed on CPU resources with 32GB of RAM.

For the Chatbot Arena dataset, we sampled responses from Gemini 1.5 Flash~\cite{geminiteam2024gemini15unlockingmultimodal}  using an adapted version of the Chatbot Arena auto-eval prompt.\footnote{\url{https://github.com/lm-sys/arena-hard-auto/blob/main/config/judge_config.yaml}}  Below is an example prompt. Color is added for clarity. Ten responses from Gemini 1.5 Flash are then sampled, with five responses using the same prompt with the order of A and B flipped. The final label is taken as the majority vote.

\begin{tcolorbox}[graybox]
Please act as an impartial judge and evaluate the quality of the responses provided by two AI Assistants to the user prompt displayed below. You will be given Assistant A's answer and Assistant B's answer. Your job is to evaluate which assistant's answer is better. \\
\\
When evaluating the assistants' answers, first identify any mistakes or inaccurate information. Next, consider if the assistant's answers are helpful, relevant, and concise. Helpful means the answer correctly responds to the prompt or follows the instructions. Note that when the user prompt has any ambiguity or more than one interpretation, it is more helpful and appropriate to ask for clarifications or more information from the user than providing an answer based on assumptions. Relevant means all parts of the response closely connect or are appropriate to what is being asked. Concise means the response is clear and not verbose or excessive. Then consider the creativity and novelty of the assistant's answers when needed. Finally, identify any missing important information in the assistants' answers that would be beneficial to include when responding to the user prompt. \\
\\
<|User Prompt|> \\
\textcolor{blue}{make a haiku on bacon the food} \\
<|The End of User Prompt|> \\
\\
<|The Start of Assistant A's Answer|> \\
\textcolor{brown}{Crisp strips of delight, \\
Sizzling dance, morning's first light, \\
Bacon whispers, "Bite."} \\
<|The End of Assistant A's Answer|> \\
\\
<|The Start of Assistant B's Answer|> \\
\textcolor{purple}{
Here is a haiku about bacon: \\
\\
Sizzling in pan \\
Savory salty bacon strips \\
Crispy delight yum} \\
<|The End of Assistant B's Answer|> \\
\\
Is the higher quality response: \\
(A) Assistant A is better \\
(B) Assistant B is better \\
Please answer with either (A) or (B).
\end{tcolorbox}

For  $G$, we finetune Gemma-3 4B for two hours on TPUv3 resources. The same prompt is used, however, we maximize the log-likelihood of the target Gemini-based answer used for $H$ instead of sampling. Early stopping is done based on the validation loss at predicting $H$ using a held-out split of the training data (recall that the training data is composed of other model comparisons from the Chatbot Arena dataset that are distinct from the one on which we evaluate our annotation policies).\looseness=-1

\end{document}